\documentclass{article}
\pdfoutput=1

\usepackage[utf8]{inputenc} 
\usepackage[T1]{fontenc}    

\usepackage{amsmath, amsfonts, amssymb, amsthm, amsbsy, amscd, bm, bbm,mathrsfs}
\usepackage{nicefrac}       

\usepackage{graphicx}       
\usepackage{caption}
\usepackage{subcaption}     
\usepackage{booktabs}       

\usepackage{hyperref}       
\usepackage{url}            

\usepackage{microtype}      
\usepackage{xcolor}         
\usepackage{enumitem}       
\setlist{leftmargin=5mm}    
\usepackage{wrapfig}        

\usepackage{cleveref}       
\usepackage{multirow}       
\usepackage{titlesec}       

\usepackage{hyperref}

\newcommand{\Note}[1]{}
\renewcommand{\Note}[1]{\hl{[#1]}}

\graphicspath{{figs/}}

\newcommand{\R}{\mathbb{R}}

\newtheorem{othertheorem}{othertheorem}[section]
\newtheorem{lemma}[othertheorem]{Lemma}

\newtheorem{claim}[othertheorem]{Claim}

\theoremstyle{definition}
\newtheorem{definition}[othertheorem]{Definition}
\newtheorem{remark}[othertheorem]{Remark}

\theoremstyle{definition}

\newcommand{\eff}{{\rm eff}}
\newcommand{\dip}{{\rm dp}}
\newcommand{\E}{\mathbb{E}}
\newcommand{\Var}{\mathrm{Var}}

\newcommand{\Comment}[1]{\hfill {\(\triangleright\) #1}}

\colorlet{darkgreen}{green!65!black}
\colorlet{darkblue}{blue!75!black}
\colorlet{darkred}{red!80!black}
\definecolor{lightblue}{HTML}{0071bc}
\definecolor{lightgreen}{HTML}{39b54a}

\definecolor{shilv}{HTML}{16A951}

\newcommand{\g}{\bm{g}}

\newcommand{\var}{\operatorname{Var}}

\usepackage[accepted]{icml2024}

\theoremstyle{plain}

\usepackage[textsize=tiny]{todonotes}

\icmltitlerunning{Delving into Differentially Private Transformer}

\begin{document}

\twocolumn[
\icmltitle{Delving into Differentially Private Transformer}



\icmlsetsymbol{equal}{*}

\begin{icmlauthorlist}
\icmlauthor{Youlong Ding}{szu,huji}
\icmlauthor{Xueyang Wu}{hkust}
\icmlauthor{Yining Meng}{comp}
\icmlauthor{Yonggang Luo}{comp}
\icmlauthor{Hao Wang}{ru}
\icmlauthor{Weike Pan}{szu}
\end{icmlauthorlist}
\icmlaffiliation{huji}{The Hebrew University of Jerusalem, Jerusalem, Israel}
\icmlaffiliation{ru}{Rutgers University, New Jersey, USA}
\icmlaffiliation{szu}{College of Computer Science and Software Engineering, Shenzhen University, Shenzhen, China}
\icmlaffiliation{hkust}{Hong Kong University of Science and Technology, Hong Kong SAR, China}
\icmlaffiliation{comp}{Changan Automobile; Changan Technology Co., Ltd, Chongqing, China}

\icmlcorrespondingauthor{Weike Pan}{panweike@szu.edu.cn}

\icmlkeywords{Machine Learning, ICML}

\vskip 0.3in
]



\printAffiliationsAndNotice{}  

\begin{abstract}
Deep learning with differential privacy (DP) has garnered significant attention over the past years, leading to the development of numerous methods aimed at enhancing model accuracy and training efficiency. This paper delves into the problem of training Transformer models with differential privacy. Our treatment is modular: the logic is to `reduce' the problem of training DP Transformer to the more basic problem of training DP vanilla neural nets. The latter is better understood and amenable to many model-agnostic methods. Such `reduction' is done by first identifying the hardness unique to DP Transformer training: the attention distraction phenomenon and a lack of compatibility with existing techniques for efficient gradient clipping. To deal with these two issues, we propose the Re-Attention Mechanism and Phantom Clipping, respectively. We believe that our work not only casts new light on training DP Transformers but also promotes a modular treatment to advance research in the field of differentially private deep learning.
\end{abstract}

\section{Introduction}
Differential privacy~\cite{dwork2006calibrating,dwork2014algorithmic} has been the gold standard for quantitative and rigorous reasoning about privacy leakage from the processing of private data. Applying differential privacy to machine learning~\cite{song2013stochastic,bassily2014private,AbadiCGMMT016} is a long-lasting challenge due to the dramatic reduction in the model utility compared with the non-private version. 
This motivates numerous studies dedicated to designing private learning algorithms without sacrificing too much utility~\cite{TramerB21,sajadmanesh2021locally,DBLP:conf/ccs/KolluriBHS22,10179409}. Meanwhile, the Transformer model~\cite{vaswani2017attention} has emerged as a versatile and effective architecture with broad applications. This paper delves into the problem of training Transformer models with differential privacy.

\begin{figure}[ht]
\centering
\setlength{\abovecaptionskip}{0.cm}
\includegraphics[width=5cm]{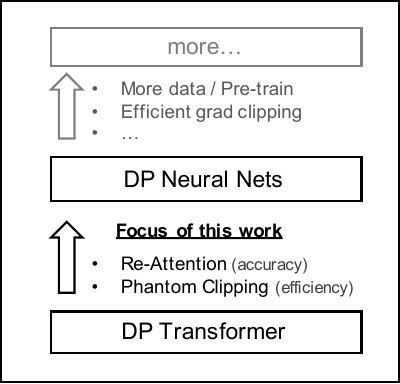}

\caption{The modular treatment in this work, where the focus of this work is on the first `reduction'.
}
\label{fig:modular}
\vspace{-0.2cm}
\end{figure}
\textbf{Hardness of training DP neural nets.}
Let us first take a step back to consider the hardness of training vanilla neural nets with differential privacy in general\footnote{For concreteness, throughout this paper, the reader can think of `vanilla neural nets' as models like MLPs.}, which encapsulates common challenges across neural network architectures.
The empirical hardness of training neural nets with differential privacy is indicated by a large body of recent works, which rely heavily on the use of large pre-trained models~\cite{li2022large} or extensive training data~\cite{de2022unlocking}. 
On the theoretical side, such hardness is directly seen and aligns well with the empirical results. Put informally, differentially private machine learning, based on its sample complexity~\citep{dwork2009complexity}, necessitates a \emph{sufficiently} larger (than non-private version) volume of data to learn patterns without resorting to the memorization of individual data points~\citep{carlini2019secret,feldman2020does}. 

While the paradigm of non-private pre-training and DP finetuning seems promising in the terms of model accuracy, particularly in domains such as image classification~\citep{tramerdifferentially,golatkar2022mixed,de2022unlocking} and natural language processing~\citep{yu2022differentially,li2022large,he2023exploring}.
Relating it to the theoretical hardness mentioned above, it serves as a method to overcome the increased sample complexity imposed by differential privacy. More specifically, collecting more training data is explicit, while pre-training is implicit in that the large amount of data is absorbed into the pre-trained weights.

However, we believe that such paradigm is not likely to be a holy grail solution. 
We refer the reader to the recent criticism against this paradigm~\cite{tramer2022considerations}. 
We also stress that it does not satisfy differential privacy (at least with respect to the public data\footnote{It is not provable that their methods provide a DP guarantee to the private dataset unless making additional (potentially strong) assumptions with the flavor like: the private dataset and public dataset are (at least computationally) independent. The contradiction is that if they were, we would not hope to gain anything by utilizing them (efficiently).}). In fact, if we look back at the history of the development of privacy protection techniques, it turns out there are many heuristic methods which were once believed to be privacy-preserving but later proved to violate the privacy in some unexpected way. One prime example is the notable deanonymization of the Netflix dataset~\cite{narayanan2008robust,narayanan2009anonymizing}.
Besides that, we also argue that the less idealized scenario, where no large-scale pre-trained model/training data is available, happens often in reality, which prevents this paradigm coming into play in the first place. One example is the need for privacy-preserving commercial recommender systems, where a machine learning model is trained differentially privately on users' historical interactions for sequential prediction of users' behaviors. It is clear that, for such domain specific and non-linguistic task, large-scale public datasets or pre-existing pre-trained models do not trivially exist. 

Given the above considerations, we argue that the paradigm of non-private pre-training and DP finetuning should be used only as a last resort. 
This motivates us to dive deeper into differentially private deep learning and study it in a more fine-grained manner.
In this work, we restrict our attention to training Transformers~\cite{vaswani2017attention} with differential privacy.
More concretely, we will focus on the most canonical case, where the model takes as input a sequence of discrete tokens and outputs the prediction for the next token.

Different from most of previous works, our treatment is \emph{modular}. As illustrated in \Cref{fig:modular}, the overall logic is to first `reduce' such specific problem as training DP Transformer models to the more basic problem of training vanilla neural nets with DP, \emph{if there exists a gap between them}. Our main focus of this work is thus on this `reduction'.
Note that this is useful because the reduced problem, improving differentially private deep learning model-agnostically, is a better understood one, amenable to a variety of off-the-shelf model-agnostic techniques~\cite{asi2021private,shamsabadi2021losing,MohapatraSH0022,li2022private,wei2022dpis,wang2021dplis,park2023differentially}.
We now briefly outline the \emph{gap}, i.e., the additional hardness of training DP Transformers.

\textbf{Hardness of training DP Transformers.} 
When discussing the hardness of training DP Transformers, our focus will exclusively be on the unique challenges not commonly encountered in other models. This particular aspect remains largely unexplored to date.
We show that there is some inherent hardness for training Transformers with differential privacy. 
In particular, we reveal the  \emph{attention distraction} phenomenon, where the attention score, computed in the self-attention mechanism through the forward propagation in each training iteration, will be distorted \emph{maliciously}. The degree of distraction depends on the training data distribution and the privacy level of DP.
Specifically, we show that the greater the imbalance in the training data\footnote{Note that real-world data typically follows a long-tailed distribution, where a small fraction of data pattern occur frequently but the majority of data pattern appear infrequently.} and the higher the level of privacy, the more severe such distraction will be. 
As a consequence, the `learning' process of the neural network will be heavily interfered, leading to suboptimal model accuracy when the training is finished.

The other hardness that we identify lies mainly in the efficiency issue. Recall that per-sample gradient is required to bound the sensitivity of gradient in each training step, which can now be obtained for general neural networks without much overhead due to a line of works~\cite{goodfellow2015efficient,li2022large,bu2023automatic}. This is done by eliminating the need for instantiating per-sample gradient.
However, such technique cannot be applied to the standard Transformer architecture, where the challenge stems from the existence of non-plain feed-forward (and likewise, backward propagation) topology caused by embedding sharing\footnote{Models with shared embedding layer refers to binding the parameter for the input embedding layer and the output embedding layer, which is the standard practice of training transformer or other embedding-based models.}.
A simple and tempting workaround is to adopt a non-standard Transformer by instantiating (different) parameters for each embedding layer. This allows us to directly apply the existing efficient gradient clipping method, but such solution is not satisfactory considering the increased parameter number and potentially worse generalization ability due to the lack of the inductive bias by embedding sharing.

\textbf{Contributions.} Our contributions are three-fold:
\begin{itemize}
    \item Our first contribution is primarily philosophical and methodological.
    We introduce the \emph{modular treatment} for improving deep learning with differential privacy, illustrated in \Cref{fig:modular}. Given the complication of the integration of deep learning and differential privacy, we believe such modular treatment provides a systematic way to design methods that improve differentially private deep learning.
    
    \item For the technical contribution, we reveal the attention distraction phenomenon during the DP training of the Transformer. To mitigate this issue, we then propose the Re-Attention Mechanism.
    The key techniques underlying the Re-Attention Mechanism is inspired by the field of Bayesian deep learning~\cite{BDL,wang2020survey}. Essentially, what is required is to track each layer's output distribution~\cite{wang2016natural} during neural networks' forward propagation. 
    We use techniques from Bayesian deep learning to efficiently approximate the output distribution of each layer by propagating the natural parameters through each layer. We note that the use of such techniques in this work is only made in a black-box manner. Thus it can be replaced by other more sophisticated techniques to serve the goal. 

    \item Our second technical contribution is a small trick for obtaining the per-sample gradient norm without instantiating per-sample gradient, dubbed Phantom Clipping. Our Phantom Clipping generalizes the technique from~\citet{li2022large}, which is inspired by ~\citet{goodfellow2015efficient}. In particular, we support standard Transformer models by providing compatibility with embedding sharing. The consequence is that it we now enjoy the full advantage of efficient gradient clipping (without compromising generalization ability caused by using a non-standard Transformer architecture), just as other vanilla neural nets like MLPs. 
\end{itemize}

\section{Preliminaries}

\begin{definition}
\textbf{$(\varepsilon, \delta)$-Differential Privacy (DP)~\citep{dwork2006calibrating,dwork2014algorithmic}:} A randomized mechanism $\mathcal{M}:\mathcal{D}\rightarrow \mathcal{R}$  satisfies $(\varepsilon, \delta)$-differential privacy if for any two datasets $\mathcal{D}, \mathcal{D}'\in {\rm Domain}(\mathcal{M})$ that differ in one record and for all $S \in {\rm Range(\mathcal{M})}$ it holds that $\operatorname{Pr}(\mathcal{M}(\mathcal{D}) \in \mathcal{S}) \leq e^{\varepsilon} \operatorname{Pr}\left(\mathcal{M}(\mathcal{D}') \in \mathcal{S}\right)+\delta$.
\end{definition}
One desirable property of DP is that it ensures privacy (in terms of $\varepsilon$ and $\delta$) under composition.
Based on this property, DP-SGD~\citep{AbadiCGMMT016} injects calibrated Gaussian noise into model gradients in each training step to achieve differential privacy as follows,
\begin{equation}
    \bm{G} = \frac{1}{B}\sum_{i=1}^B \g_i\cdot \operatorname{Clip_C}(\|\g_i\|) + \sigma_{\dip} \cdot \mathcal{N}\left(0, \mathbf{I}\right),
\label{eq:DPSGD}
\end{equation}
where $G$ is the averaged gradient among the minibatch, $\g_i$ is the gradient of the $i$-th sample in the minibatch of size $B$, $C$ is the clipping norm, $\operatorname{Clip_C}(\|g_i\|)=\operatorname{min}(C/\|g_i\|, 1)$, ensuring that the sensitivity of the averaged gradient $G$ is bounded by $\Delta_G \leq \|g_i\cdot \operatorname{Clip}(\|g_i\|) \| \leq C$. $\sigma_\dip$ is the noise multiplier derived from privacy accounting tools~\citep{balle2018privacy,wang2019subsampled}.


The detailed discussion of related work is in \Cref{apdx:related}.

\section{Phantom Clipping}

In order to bound the sensitivity of individual training sample, DP-SGD (\Cref{eq:DPSGD}) clips the per-sample gradient for each training sample in the minibatch. This can be done (naively) by first computing the per-sample gradient for each sample, followed by a gradient clipping. The downside is that this will incur significant memory overhead since we need to instantiate per-sample gradient for each training sample individually.

The method of clipping per-sample gradient  without instantiating per-sample gradient~\citep{goodfellow2015efficient} has shown considerable efficiency advantage~\citep{li2022large} for Transformer models as compared to other libraries or implementations (for instance, Opacus~\citep{yousefpour2021opacus}, JAX~\citep{subramani2021enabling}). 
By eliminating the need for instantiating per-sample gradient, it allows larger batch size and thus enjoys better parallelism.

Let us first review how this method works. Observe that the computational bottleneck of gradient clipping in \Cref{eq:DPSGD} lies in the calculation of the per-sample gradient norm i.e., $\|g_{i}\|$. As the $L_2$ norm of a vector can be decomposed cross arbitrary dimensions\footnote{For example, $\left\|(a, b, c)\right\| = \left\| (\|a\|, \|[b, c]\|)\right\|$ holds for vectors $a,b,c\in \R^*$ in arbitrary shapes.}, 
the  gradient norm for sample $i$ can be computed as
\begin{equation}
    \|g_{i}\| = \left\|\left(\left\|g_{i}^{(1)}\right\|, \left\|g_{i}^{(2)}\right\|, \dots,\left\|g_{i}^{(l)}\right\|\right)\right\|,
\end{equation}
where $l$ is the total number of layers, $g_{i}^{(j)}$ is the gradient of $j$-th layer of the neural networks for $1\leq j\leq l$.
The next step is to scale the gradient by a factor of $\operatorname{Clip}_C(\|g_{i}\|)$ to bound its sensitivity. This can either be accomplished by re-scaling the loss $\mathcal{L}_i$ through this factor, followed by a second backpropagation~\citep{lee2021scaling}, or by manually scaling the gradient as demonstrated by~\citep{bu2023differentially}.

The existing efficient gradient 
clipping methods that compute $\left\|g_{i}^{(j)}\right\|$ without instantiating $g_{i}^{(j)}$ is based on an implicit assumption that the layer $j$ will not be reused during a single forward propagation. Therefore, the incompatibility with standard Transformer models arises from the fact that the (parameter of) embedding layer will be accessed twice during one run of forward propagation: one for input embedding and the other for output embedding. Such parameter sharing leads to non-plain feed-forward (and backward propagation) topology, which makes existing methods not applicable.

\subsection{Phantom Clipping}
\begin{figure}[ht]
\centering
\setlength{\abovecaptionskip}{0.cm}
\includegraphics[width=7cm]{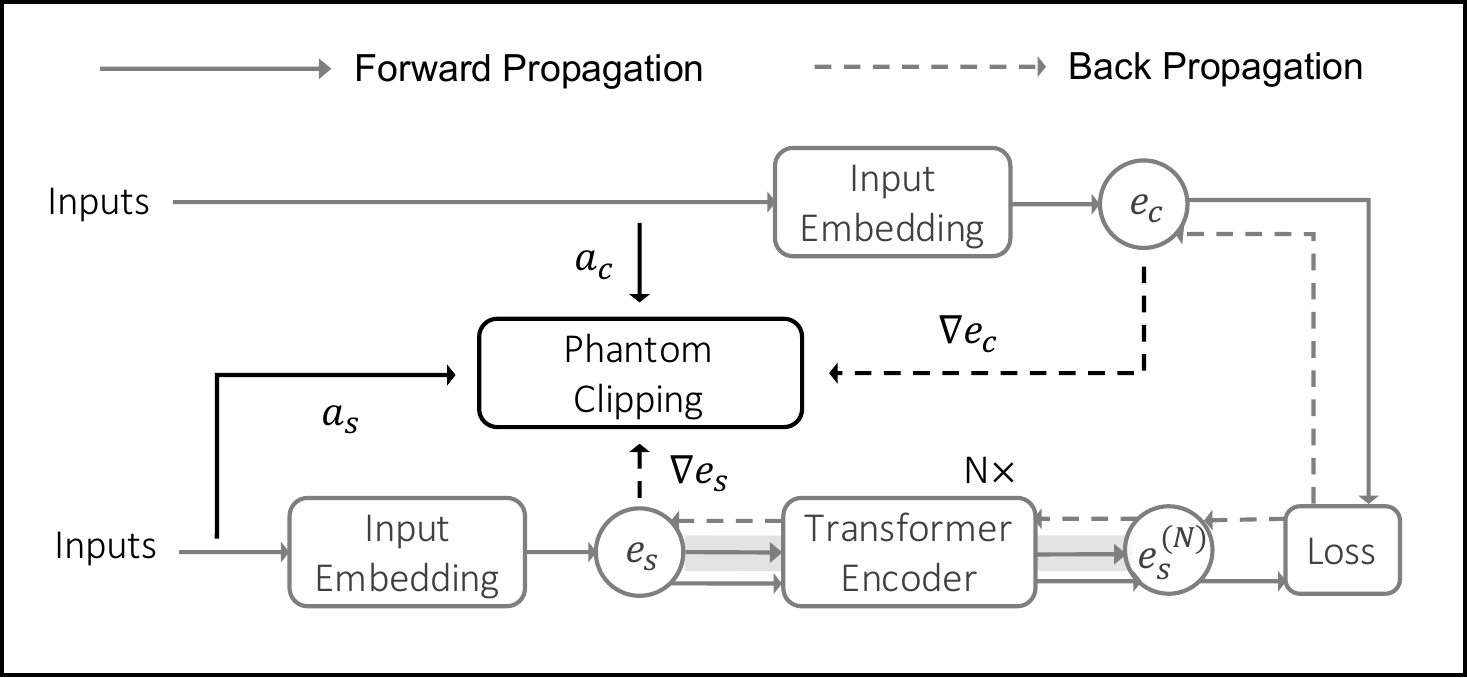}

\caption{Phantom Clipping, illustrated.
}
\label{fig:pc}
\vspace{-0.2cm}
\end{figure}
In this section, we present \emph{Phantom Clipping}, a technique for efficient private training of Transformers without the need for instantiating per-sample gradient. As discussed above, it suffices to enable the efficient gradient norm computation for the shared embedding layer, as other layers can be handled by existing methods. 

We first introduce some notations.
Let $B$ be the batch size, $L \in \mathbb{N}$ be the length of each sequence, $\mathcal{S} \in \mathbb{N}^{B\times L}$ be a minibatch of training samples,
$s_i\in \mathbb{N}^{L}$ be $i$-th training sample in a minibatch for $i\in [B]$. 
Let $M \in \mathbb{N}$ be the vocabulary size, $E\in \mathbb{R}^{M\times d}$ be the (shared) embedding layer where $d$ is the dimension of each embedding. 
We group the parameter of whole network into two parts, $\theta = (E, -E)$.
Let $(a_{\mathcal{S}})_i \in \{0, 1\}^{L\times M}$ be the one-hot encodings of the input sequence $s_i$ for $i\in [B]$. When it is fed into the input embedding layer, the output $(e_{\mathcal{S}})_i\in \R^{L\times d}$ is,
\begin{equation}
    (e_{\mathcal{S}})_i = \text{InputEmbedding}_E((a_{\mathcal{S}})_i).
\end{equation}
Let $(\tilde{e}_{\mathcal{S}})_i \in \mathbb{R}^{L\times d}$ be the output of the Transformer encoder when  it is fed into $(e_{\mathcal{S}})_i$, i.e.,
\begin{equation}
    (\tilde{e}_{\mathcal{S}})_i = \text{TransformerEnc}_{-E}((e_{\mathcal{S}})_i)
\label{eq:transformerenc}
\end{equation}
Let $\mathcal{C} \in \mathbb{N}^{M'}$ be the candidates for the next prediction where $M'$ is the number of candidates and $M' = M$.
Let $(a_{\mathcal{C}})_j \in \{0, 1\}^{M}$ be the one-hot encoding of $j$-th candidate in $\mathcal{C}$ for $j\in [M']$. When it is fed into the output embedding layer, the output $e_{\mathcal{C}}\in \R^{M'\times d}$ is,
\begin{equation}
    (e_{\mathcal{C}})_j = \text{OutputEmbedding}_E((a_{\mathcal{C}})_j).
\end{equation}
For the training sample $i\in [B]$, the candidate $j \in [M']$, the prediction score $(r_{\mathcal{S}})_{i, j} \in \mathbb{R}$ for the next token of the sentence $({\mathcal{S}})_{i}$ is obtained by
\begin{equation}
    (r_{\mathcal{S}})_{i, j} = \langle \left(e_{\mathcal{C}}\right)_j, (\tilde{e}_{\mathcal{S}})_i\rangle,
\end{equation}
where $(\cdot)_{i,j}$ denotes the $(i,j)$ entry of a matrix.

Let $\mathcal{L} = \frac{1}{B}\sum_{i=1}^B\mathcal{L}_i$ be the average loss in the minibatch where $\mathcal{L}_i$ is the per-sample loss with respect to the $i$-th sample, which is computed by the cross-entropy loss,
\begin{equation}
    \mathcal{L}_i = \text{CrossEntropy}((r_{\mathcal{S}})_{i}, (y_{\mathcal{S}})_i),
\end{equation}
where $(y_{\mathcal{S}})_i \in \{0, 1\}$ is the ground truth, i.e., the true next token of the sequence $s_i$. A standard back-propagation gives us the following intermediates for $i\in [B]$,
\begin{equation}
    \nabla (e_{\mathcal{S}})_i:= \partial \mathcal{L}_i / \partial (e_{\mathcal{S}})_i \in \mathbb{R}^{L\times d},
\end{equation}
\begin{equation}
     (\nabla e_{\mathcal{C}})_i := \partial \mathcal{L}_i / \partial e_{\mathcal{C}} \in \mathbb{R}^{M'\times d}.
\end{equation}
Given above entities, we show how to obtain the per-sample gradient norm for the shared embedding $ \|g_{E}\|$ in \Cref{alg:pc}.
The derivation is deferred to \Cref{apdx:pfphan}. 

\begin{algorithm}[!htb]
    \caption{Phantom Clipping}
      \textbf{Parameter}: Bach size $B$, sentence length $L$, vocabulary size $M$, candidate size $M'=M$, embedding dimension $d$.\\
     \textbf{Input}: $a_{\mathcal{S}} \in \mathbb{R}^{B\times L\times M}$, $a_{\mathcal{C}} \in \mathbb{R}^{M'\times M}$, $\nabla e_{\mathcal{S}} \in \mathbb{R}^{B\times L\times d}$, $\nabla e_{\mathcal{C}} \in \mathbb{R}^{B\times M'\times d}$.

     \textbf{Notation}: For $X\in \mathbb{R}^{a\times b\times c}, Y\in \mathbb{R}^{a\times c\times d}$, define $X\cdot Y \in \mathbb{R}^{a\times b\times d}$ as $(X\cdot Y)_i = X_i \cdot Y_i \in \mathbb{R}^{b\times d}$ for $i\in [a]$.
    
    \begin{algorithmic}[1] 


    
	\STATE Calculate $A \leftarrow a_{\mathcal{S}}\cdot a_{\mathcal{S}}^T \in \mathbb{R}^{B\times L\times L}$;
	\STATE Calculate $B \leftarrow \nabla e_{\mathcal{S}}\cdot \nabla e_{\mathcal{S}}^T \in \mathbb{R}^{B\times L\times L}$;
	\STATE Calculate $C \leftarrow a_{\mathcal{S}}\cdot \nabla e_{\mathcal{C}} \in \mathbb{R}^{B\times L\times d}$;
    \FOR {$i= 1, ..., B$ (in parallel)} 
    \STATE Calculate $\|g_E\|_i (\in \mathbb{R}) \leftarrow$ \\$ \left(\langle A_i, B_i\rangle^2 + \|(\nabla e_{\mathcal{C}})_i\|^2 + 2\cdot \langle \nabla (e_{\mathcal{S}})_i, C_i \rangle\right)^{1/2} $;
     \ENDFOR
     \STATE Output per-sample gradient norm $\|g_{E}\| \in \mathbb{R}^{B}$.
    \end{algorithmic}
\label{alg:pc}
\end{algorithm}

\noindent \textbf{Privacy guarantee}. Observe that what we propose is an efficient way to compute a function (i.e., per-sample gradient norm) required by DP-SGD. The input-output behavior of our method is \emph{identical} to DP-SGD. Therefore, our method inherits its privacy guarantee.

\noindent \textbf{Discussion on memory overhead for the embedding layer.} We study the additional memory footprint required by computing the gradient norm for embedding layer as above. 
Step 1 has memory complexity of $O(BL^2)$, since the space complexity for multiplying two matrices $X \in \mathbb{R}^{m\times n}$ and $Y \in \mathbb{R}^{n\times p}$ is roughly $O(mp)$ if implemented appropriately, which is the case for frameworks like PyTorch.
The memory complexity of step 2 and 3 follows for the same reason. Step 4 involves only in-place operations.
Therefore, the total memory complexity is $O(BL^2 + BLd)$.

As a comparison, the existing gradient clipping method for Transformer~\cite{li2022large}, i.e., Ghost Clipping, has a memory complexity of $O(BT^2)$ when the input to the layer $a_i$ has the shape of $\mathbb{R}^{T\times \cdot}$. Hence, its memory complexity for the two embedding layers is $O(BM^2 + BL^2)$ where $M$ is the vocabulary size and $L$ is the sentence length. 
For concreteness, the reader can think of $L > d$ and $L = o(M)$ since the sentence length is typically significantly less than the vocabulary size.
Since Ghost Clipping is not specifically optimized for the embedding layer, by a refinement of their method, it is possible to optimize it to $O(BL^2)$\footnote{Our contribution is that we first refine the technique of Ghost Clipping to improve its training speed for the embedding layer. Starting from that, we further make it to support the embedding sharing, resulting in our Phantom Clipping.}. Therefore, our Phantom Clipping matches this memory complexity while additionally supporting embedding sharing for the standard Transformer models. Recall that embedding sharing has benefits beyond efficiency, such as better generalization.

\noindent \textbf{Discussion on overall speedup.} 
By our Phantom Clipping, we are able to compute the gradient norm for the shared embedding layer without instantiating the per-sample gradient. 
Note that the total memory overhead comprises multiple additive components, each corresponding to a specific layer, i.e., $O($cost of the embedding layer + cost of other layers$)$. 
We have discussed the cost of the embedding layer. The cost of other layers remains the same as existing methods such as Ghost Clipping. Hence, the advantage discussed above might diminish when the costs associated with other layers dominate the overall term. However, when the model is relatively small, the cost for the embedding layer will be the dominant term, making Phantom Clipping significantly more advantageous.
We note that such small model is indeed preferred in reality in the sense that it is suitable for local inference on computing-restrained end devices, eliminating the need for uploading sensitive data to the cloud service for inference~\cite{ramaswamy2020training}. Empirical evaluation of Phantom Clipping is reported in \Cref{apdx:phanempirical}.

\section{Re-Attention Mechanism}
\label{sec:reattn}

We present the Re-Attention Mechanism in Section 4.3. Before that, to provide a glimpse into the attention distraction phenomenon, we summarize the high-level intuition in Section 4.1 and explain in more detail in Section 4.2. If the reader prefers to read the actual method directly, please start from Section 4.3.

\subsection{Overview and Intuition}
\label{sec:intuition}
We will first provide an overview and some intuition behind the attention distraction phenomenon.

\textbf{The anisometry of the embedding parameter.}
Recall that DP-SGD (\Cref{eq:DPSGD}) adds isometric Gaussian noise to the model parameters. We might be tempted to think that each parameter is treated equally. However, the situation is a bit more nuanced, in particular, when the input space is discrete.
A key observation is that, when training Transformer models with differential privacy, the parameter of the embedding layer is anisometric in the sense that some parameters are relatively well trained, corresponding to small variance, while some are not, corresponding to large variance.
To see this, let us consider a thought experiment as follows. Let the embedding layer be parameterized by $E\in \R^{M\times d}$ where $M$ is the vocabulary size and $d$ is the dimension of the embedding. 
Suppose we unintentionally add a dummy embedding $e_{M+1}\in \R^d$ into $E$ to form $E'\in \R^{(M+1)\times d}$, and start training the model with differential privacy as usual. 

During each iteration $t$, isometric Gaussian noise is added to the whole embedding layer $E'$ to privatize the gradient. After $T$ iterations, the value of that dummy embedding is
\begin{equation}
\begin{split}
    e_{M+1} = e_0 + \epsilon_1 + \epsilon_2 + \dots + \epsilon_T
\end{split}
\end{equation}
where $e_0$ is the initial value of $e_{M+1}$, $\epsilon_i \sim \mathcal{N}(0, \sigma_{\dip} \mathbf{I})$ for $1\leq i \leq T$. $\sigma_{\dip}$ is the noise multiplier computed from the privacy accountant.
Therefore, $e_{M+1}$ is distributed as $\mathcal{N}(e_0, T\sigma \cdot \mathbf{I})$.
In other words, after each iteration the variance gets larger and hence the embedding gets noisier and less informative (if the initial value $e_0$  contains some information). 

With that in mind, let us then consider a token that infrequently 
occurs (i.e., close to dummy in that sense). If it does not occur in the current training batch, its embedding will not gain information from the gradient, but the Gaussian noise is always injected\footnote{Note that this is necessary to ensure differential privacy, which hides the membership of this token.}. The noise will be accumulated after several iterations until this token occurs in the training batch.
Different tokens have different frequency, which leads to anisometric variance in their corresponding embedding.

\textbf{Nonlinearity of attention computation.} 
So, what will happen if tokens have anisometric variance? Suppose we apply some linear transformation $T: x \mapsto ax+b$ to the random variable $x$, this is acceptable in the sense that $\E[T] = T(\E[x])$. In other words, the expected value after the transformation is \emph{independent of its variance}.

However, the attention computation is nonlinear, which involves exponentiation operation. Think of the attention computation as a function $\text{Attn}(\cdot)$ which takes as input $x$, outputs its attention score with respect to some query. If the input $x$ follows the distribution $\mathcal{N}(\mu, \sigma)$, then the expected output value can be represented as $\E[\text{Attn}(x)] = f(\sigma, \text{Attn}(\mu))$, where $f(\cdot)$ is some function that monotonically increases with $\sigma$. To put it in context, suppose we are computing the attention scores of the tokens in a sentence. Then tokens with higher variance will attract more attention than it should. In other words, the attention is distracted maliciously, which turns out to be in favor of the tokens with higher variance, regardless of its actual relevance. Details follow.

\subsection{Attention Distraction}
\label{sec:attentiondistraction}
For ease of exposition, we split the randomness required by training Transformer model with differential privacy into two parts. The coin flipping $r_1\in \{0, 1\}^{*}$ is inherited from non-private machine learning, which involves model initialization and minibatch sampling. 
The coin flipping $r_2\in \{0, 1\}^{*}$ is additionally required for injecting DP noise. We fix $r_1$ and let $\bm{r_2}$ be random variable uniformly distributed over $\{0, 1\}^{*}$. 
Our subsequent analysis holds for arbitrary $r_1$, and therefore, it is also valid for uniformly random values of $r_i$ by an average argument.
For clear exposition, we will use bold font to represent a random variable and normal font to represent the value taken by the corresponding random variable.

For $2\leq t \leq T$, at the end of the iteration $t-1$ , conditioned on the event that the model parameter before adding DP noise is $\theta_{\text{non\_priv}}^{(t-1)}$, DP noise is injected to privatize the model as follows, which is a re-formulation of DP-SGD,
\begin{equation}
    \bm{\theta^{(t-1)}} = \text{Privatize}\left(\theta_{\text{non\_priv}}^{(t-1)}, \bm{r_2}\right),
\label{eq:privatize}
\end{equation}
where $\text{Privatize}(\cdot)$ refers to the process of adding Gaussian noise to the model parameter for differential privacy according to \Cref{eq:DPSGD}. More concretely, it takes as input the non-private version of the updated model parameter $\theta_{\text{non\_priv}}^{(t-1)}$ and the randomness $\bm{r_2}$ required for sampling Gaussian noise, then injects the noise into the model parameter\footnote{Note that this is equivalent to \Cref{eq:DPSGD}, which first adds noise to the gradient and updates the model parameter.}. The distribution of random variable $\bm{\theta_{(t-1)}}$ is induced by the uniform random variable $\bm{r_2}$. Namely, $\bm{\theta_{(t-1)}}$ follows a Gaussian distribution with mean $\theta_{\text{non\_priv}}^{(t-1)}$ and variance $\sigma_{\dip}$. It is important to stress that we are not allowed to access $\theta_{\text{non\_priv}}^{(t-1)}$, otherwise this could lead to the violation of differential privacy. Instead, what we can access is the noisy parameter $\theta_{\text{private}}^{(t-1)}$ sampled from $ \bm{\theta^{(t-1)}}$ (but we can only sample once).

For simplicity and without loss of generality, let the batch size be 1. At the beginning of the iteration $t$, the training process will first sample a training sequence using $r_1$ as the randomness, denoted by $s \in \mathbb{N}^L$ where $L \in \mathbb{N}$ is the length of the sequence. It then feeds $s$ into the model for forward propagation. Before performing the attention score calculation, it will compute the key for token $i$. The whole process can be represented as

\begin{equation}
\begin{split}
    \bm{K}_i = \text{PreAttention}\left(\bm{e_i}, \bm{\theta^{(t-1)}}_{-E}, \mathcal{D}, r_1\right), \\
\end{split}
\end{equation}
where $\bm{e_i}\in \R^{d}$ is the embedding of $i$-th token, $\bm{\theta^{(t-1)}}_{-E}$ is the parameter of the Transformer encoder (i.e., excluding the parameter of the embedding layer), $\mathcal{D}$ is the training dataset, $\bm{K}_i \in \mathbb{R}^{d\times L}$ ($d \in \mathbb{N}$ is the model dimension) is the random variable representing the keys waiting for attention score calculation. Its distribution is induced by $\bm{e_i}, \bm{\theta^{(t-1)}}_{-E}$. Qualitatively, observe that the higher the input variance $\Var[\bm{e_i}]$, the higher 
the output variance  $\Var[\bm{K}_i]$ will be. In other words,  the heterogeneity of the embedding parameter translates to the heterogeneity of its key.

For a query $q\in \R^{d}$, let $\bm{S}_i$, $1\leq i\leq L$, be random variable taking values in $\mathbb{R}$, which represents the attention score of $i$-th token in the training sequence $s$, computed from $q$ and $\bm{K}_i$. With some basic algebraic manipulation and applying the theory of extreme value~\citep{coles2001introduction}, we can recast the formula for attention scores as follows\footnote{For ease of notation, we omit the constant factor (i.e., $1/\sqrt{d}$) in attention computation.},
\begin{equation}
\begin{split}
    \bm{S}_i &= \text{Attention}(q, \bm{K}_i) \\
     & = \frac{\exp{\langle q, \bm{K}_i\rangle}}{\sum_{j=1}^L \exp{\langle q, \bm{K}_j\rangle}} \\
    &= \exp{\left(\langle q, \bm{K}_i\rangle - \log \sum_{j=1}^{L} \exp{\langle q, \bm{K}_j\rangle}\right)} \\
    &= \exp{\left(\langle q, \bm{K}_i\rangle - \E_{\bm{\gamma}} [\operatorname{max}_j\{\langle q, \bm{K}_j\rangle + \bm{\gamma} \}]\right)}, \\
\end{split}
\label{eq:logsumexp}
\end{equation}
where $\bm{\gamma}$ is distributed as a standard Gumbel. 
Let us consider some token $i' \in [L]$ that should have attracted little attention given the query $q$,
then the expectation of the noisy maximum $\E_{\bm{\gamma}} [\operatorname{max}_j\{\langle q, \bm{K}_j\rangle + \bm{\gamma} \}]$ can be approximated by $\widetilde{\bm{M}} = \operatorname{max}_{j\neq i'}{\langle q, \bm{K}_j\rangle} + \zeta$, where $\zeta = \E [\gamma] = \pi^2/6$.
Taking the expectation of \Cref{eq:logsumexp} over $\bm{K}_{i'}$, by Jensen's Inequality, it holds that 
\begin{equation}
\E_{\bm{K}_{i'}} [\bm{S_{i'}}] \geq \exp\left({\E_{\bm{K}_{i'}} [\langle q, \bm{K}_{i'}\rangle - \widetilde{\bm{M}}]}\right)
\end{equation}
That is, the expected attention score for token $i'$ will be larger than it should be. To be more quantitative, we assume that most of the mass of $\bm{K}_{i'}$ is symmetric around its mean. Jumping ahead, we can conclude
\begin{equation}
\E_{\bm{K}_{i'}} [\bm{S_{i'}}] \approx f(\Var[\bm{K}_{i'}])\cdot \exp\left({\E_{\bm{K}_{i'}} [\langle q, \bm{K}_{i'}\rangle - \widetilde{\bm{M}}]}\right),
\label{eq:attndistraction}
\end{equation}
where $f(\cdot)$ is some function whose output is greater than 1 and  monotonically increases with its input.

As a result, tokens with higher variance result in inflated attention scores due to the multiplicative bias $f(\Var[\bm{K}_{i'}])$, distracting attention from more deserving tokens, given that token $i'$ is presupposed to garner little attention under query $q$. 
The greater the imbalance in the training data and the higher the level of privacy, the higher variance such distraction will be. 
If all tokens have similar variance or variance terms are negligible, the negative effects of this attention distraction are reduced. However, in less idealized scenarios, especially with real-world data, the effect of attention distraction could hinder the training of the Transformer, thereby degrading model utility.

\begin{algorithm}[!htb]
    \caption{Re-Attention Mechanism}
    \label{alg:re-attn}
    \textbf{Input}: $e_{\mathcal{S}}\in \mathbb{R}^{B\times L\times d}$: A minibatch of training sentences (in the form of embedding); $B$: batch size;  $L$: sentence length; $N$: number of Transformer layers\\
    \textbf{Parameter}: Privacy parameters $\varepsilon, \epsilon>0, \delta \in(0, 1)$;
    \begin{algorithmic}[1] 
	\STATE Calculate $\sigma_{\text{DP}} \leftarrow \text{PrivacyAccountant}(\delta, \varepsilon, ...)$;
	\STATE $c^{(0)}, d^{(0)} \leftarrow \text{Setup}(\sigma_{\text{DP}}, B, p)$; \Comment{\Cref{subsubsec:ErrorInstantiation}}
        \STATE Initialize $e_{\mathcal{S}}^{(0)} \leftarrow e_{\mathcal{S}}$;
        \FOR{$\ell = 1, ..., N$}
          \STATE $X^{(\ell+1)} \leftarrow \text{PreAttenionFeedForward}\left(e_{\mathcal{S}}^{(\ell-1)}\right)$;
                \STATE $S \leftarrow \text{Self-Attention}(X^{(\ell+1)})$;
                \STATE $\sigma \leftarrow \mathcal{T}\left((c, d)^{(\ell)})\right)$;
		    \STATE $S \leftarrow S / \exp \left[C\sigma^2/2\right]$;\Comment{\Cref{sec:re-attn}}
          \STATE $e_{\mathcal{S}}^{(\ell)} \leftarrow \text{PostAttenionFeedForward}(S_i)$;
          \STATE $(c, d)^{(\ell+1)} \leftarrow \mathcal{F}\left((c, d)^{(\ell)}\right)$; \Comment{\Cref{subsubsec:ErrorPropagation}}
        \ENDFOR
        \STATE Set $\tilde{e}_{\mathcal{S}} \leftarrow e_{\mathcal{S}}^{(N)}$, output $\tilde{e}_{\mathcal{S}}$.
    \end{algorithmic}
\end{algorithm}
\subsection{Re-Attention Mechanism}
At its core, the Re-Attention Mechanism is designed to mitigate the above discussed attention distraction by quantitatively keeping track of the variance. 
This logic coincides with the rationale behind Bayesian deep learning~\citep{wang2020survey}, a field primarily focused on quantifying prediction uncertainty to enhance the robustness and reliability of machine learning systems. 
While our primary interest lies in unbiased attention score computation during private training, we can leverage and adapt existing methodologies in Bayesian deep learning to achieve this distinct goal. 

The overall method is presented in \Cref{alg:re-attn} at high level, which realizes the functionality of $\operatorname{TransformerEnc}$ in~\Cref{eq:transformerenc}.
Below, in a step by step fashion, we explain each newly added procedure in detail.

\subsubsection{Setup}
\label{subsubsec:ErrorInstantiation}

Motivated by the intuition in \Cref{sec:intuition}, we naturally introduce the notion of \emph{Effective Error} to capture the heterogeneity of the variance in the embedding layer, though our actual definition encompasses all layers more broadly.
\begin{definition}
\textbf{(Effective Error)} Let $\theta \in \R^*$ be some parameter of the model, the effective error $\sigma_{\eff}^{\theta}$ associated with the model parameter $\theta$ is defined as
\begin{equation}
\begin{split}
    \sigma_{\eff}^{\theta} = \frac{\sigma_{\dip}}{B_{\eff}^{\theta}}, ~~B_{\eff}^{\theta} = \E_{\mathcal{B}\stackrel{\text{\tiny{i.i.d.}}}{~\sim~}\mathcal{D}^B} \left[\sum_{i=1}^B \mathbb{I}\left[ R_{\theta}(\mathcal{B}_i)\right]\right],
\end{split}
\label{eq:effectiveerror}
\end{equation}
where $B\in \mathbb{N}$ is the batch size, $\mathcal{B}\in \mathbb{N}^{B\times L}$ is the minibatch, i.i.d. sampled from training data distribution $\mathcal{D}$ (note that $\mathcal{B}_i \in \mathbb{N}^L$ is a sequence of tokens), $\sigma_{\dip}$ is the DP noise multiplier in \Cref{eq:DPSGD}, and $\mathbb{I\left(\cdot\right)}$ is the indicator function. $R_\theta(\cdot) = 1$ if and only if the training sentence $\mathcal{B}_i \in \mathbb{N}^L$ has relevance with $\theta$. More specifically,
we split the model parameter into $[E \in \R^{M\times d}, W = -E]$ where $E$ is the embedding layer, $W$ is the remaining. On input $s \in \mathbb{N}^L$, $R_\theta(s)$ is defined as follows,
\begin{equation}
    R_\theta(s) = \begin{cases} \mathbb{I}[{\rm token}~i\in s]&\text{if}~\theta = E_i\\ 1&\text{if}~\theta \in W,\end{cases}
\end{equation}
where $E_i$ is the $i$-entry of the matrix $E$, corresponding to $e_i$, i.e., the embedding of the token $i$.
\end{definition}

Let us parse these equations. When $\theta = W$, it holds that $B_{\eff}^{\theta} \equiv B$ and hence $\sigma_{\eff}^{\theta}=\sigma_{\dip}/B$, recovering the effective noise introduced in~\cite{li2022large}. For the embedding layer $E$, the intuition is that the more frequently a token $i$ occurs, the larger $B_{\eff}^{E_i}$ will be, leading to smaller $\sigma_{\eff}^{E_i}$. Indeed, the Effective Error for the embedding of token $i$ $\sigma_{\eff}^{E_i}$ is inversely proportional to $p_i$, which is the frequency of token $i$ (i.e., the probability of token $i$'s occurrence in data). Derivation is in \Cref{apdx:errorinst}.

\begin{claim}
    For the embedding layer $E$, the  Effective Error of token $i$ is $\sigma_{\eff}^{E_i} = \sigma_{\dip} / (B\cdot p_i)$, where $p_i$ is the frequency of token $i$ (i.e., the probability of token $i$'s occurrence in data).
\label{cl:errorinst2}
\end{claim}

To summarize, the $\sigma_{\eff}^{W}$ is directly computed from the training hyperparameters (i.e., privacy budget, batch size), and is identical for each dimension of $W$. In contrast, $\sigma_{\eff}^{E_i}$ is inversely proportional to its frequency $p_i$ and hence heterogeneous across different $i$s. 
Since obtaining the frequency $p_i$ can be easily done with differential private and is mostly orthogonal to this work. For simplicity we assume that $p_i$ is already given under DP guarantees. Alternatively, one can think of $p_i$ as some data-independent prior information.

We then turn to set up the mean for each parameter of the neural network, viewed as the random variable due to randomness $\bm{r_2}$ of the DP training.
Recall  that in the Bayesian deep learning, the model parameter is directly associated with its mean, which corresponds to  $\theta_{\text{non\_priv}}^{(t-1)}$ in \Cref{eq:privatize}. However, we can only access the noisy parameter $\theta_{\text{private}}^{(t-1)}$ after the injection of DP noise. We will use that as the approximation for its mean. Justification follows.
The process of adding Gaussian noise to ensure differential privacy is equivalent of sampling from $ \bm{\theta^{(t-1)}}\sim \mathcal{N}\left(\theta_{\text{non\_priv}}^{(t-1)}, \sigma_{\dip}\right)$. Access to this noisy parameter can be interpreted as a single sampling opportunity from its underlying Gaussian distribution, which can then be viewed as a one-time Markov chain sampling~\citep{wang2015privacy}.

\subsubsection{Error Propagation}
\label{subsubsec:ErrorPropagation}

\begin{figure}[ht]
\centering
\setlength{\abovecaptionskip}{0.cm}
\includegraphics[width=6cm]{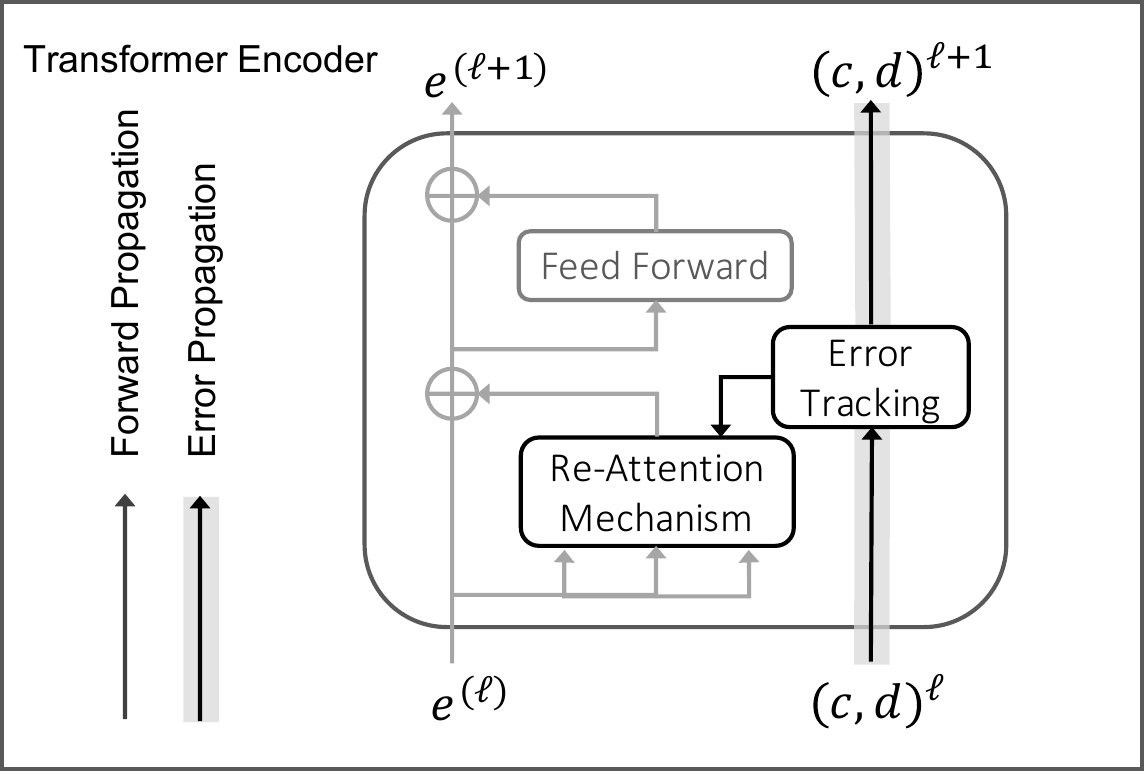}

\caption{Re-Attention Mechanism, illustrated.
}
\label{fig:reattn}
\vspace{-0.2cm}
\end{figure}

Given the effective errors of the embedding layer and of the Transformer encoder, our goal is to obtain the distribution of key $\bm{K}_{i}$ for each token $i$.
Suppose we know its exact distribution, we can then obtain its expected attention score and debias the multiplicative bias term $f(\cdot)$ in \Cref{eq:attndistraction}.
However, the exact distribution is computationally intractable. To tackle this issue, a large body of research in the field of Bayesian deep learning focuses on its efficient approximation. Details follow.
 
We denote the output of the $\ell$-th layer by the random variable $X^{(\ell)}$.
Given the output distribution $X^{(\ell-1)}$ of the preceding layer, the distribution $X^{(l)}$ can be computed layer-by-layer as follows, 
\begin{equation}
\begin{split}
    p\left(X^{(\ell)}|X^{(0)}\right) &= \E_{X^{(\ell-1)}|X^{(0)}} \left[p\left(X^{(\ell)} | X^{(\ell-1)}\right)\right].\\
\end{split}
\end{equation}
Based on Variational Inference~\citep{kingma2013auto}, we can use an approximating distribution $q$ to approximate the computationally intractable distribution $p$, where $q\left(X^{(\ell)}\right)$ follows a Gaussian distribution of mean $\mu$ and variance $\sigma^2$. 
Note that minimizing the KL divergence of $KL\left(p\left(X^{(\ell)}|X^{(0)}\right) || q\left(X^{(\ell)}\right)\right)$ reduces to matching the moments of $q\left(X^{(\ell)}\right)$ to $p\left(X^{(\ell)}|X^{(0)}\right)$. Since the mean and variance\footnote{Note that for a Gaussian distribution, (i) mean and variance, (ii) the first two moments, and (iii) natural parameter, are equivalent in the sense of mutual convertibility. We will use them interchangeably.} are sufficient statistics for Gaussian distribution, propagating the distribution reduces to propagating its natural parameters~\citep{wang2016natural}. For linear layers coupled with a coordinate-wise non-linear activation, the statistics can be computed by analytic expressions using existing techniques from Probabilistic Neural Networks~\citep{wang2016natural,shekhovtsov2019feed,gast2018lightweight,postels2019sampling,morales-alvarez2021activationlevel}. More details are deferred to \Cref{apdx:npnprop}.

All in all, we obtain the output distribution of layer $(\ell)$ via analytic expression  in terms of the natural parameter of the preceding layer's output distribution as 
\begin{equation}
(c, d)^{(\ell)} = \mathcal{F}\left((c, d)^{(\ell-1)}\right),~~~\sigma^2 = \mathcal{T}\left((c, d)^{(\ell)}\right),
\label{eq:nonprop}
\end{equation}
where $\mathcal{F}(\cdot)$ is the error propagation function and $\mathcal{T}(\cdot)$ is the conversion from natural parameter to the variance.

\subsubsection{Re-Attention} With the effective error tracked, we then proceed to mitigate the attention distraction identified in \Cref{eq:attndistraction}.
\label{sec:re-attn}
Recall that to track the intractable distribution $p$, we use an approximating distribution $q$ following some Gaussian distribution, whose parameter is obtained by moment matching.
By leveraging the fact $\E [\exp(X)] = \exp(\E[X])\exp(\operatorname{Var}[X]/2)$ for $X$ following a Gaussian distribution, the \Cref{eq:attndistraction} can be turned into
\begin{equation}
\begin{split}
    \E_{K_{i'}} [\bm{S_{i'}}] 
    &\approx  \E_{K_{i'}} [\exp{(\langle q, \bm{K}_{i'}\rangle - (\operatorname{max}_{j\neq i'}\{\langle q, \bm{K}_j\rangle\}+ \zeta))}]\\
    &= \underbrace{\exp{\left(\langle q, \bm{K}_{i'}\rangle - \widetilde{M}\right)}}_{\rm attentive~relevance} \cdot \underbrace{\exp{\left(C\sigma^2_{i'}/2\right)}}_{\rm multiplicative~error},
\end{split}
\label{eq:reattnbias}
\end{equation}

where $\widetilde{M}=(\operatorname{max}_{j\neq i'}\{\langle q, K_j\rangle\}+ \zeta)$ and the last equality leverages the fact that $\langle q, K_{i'} \rangle \sim \mathcal{N}(\langle q, k_{i'}\rangle, C\sigma^2)$ with $C=\langle q, q\rangle$.
Then we mitigate the attention distraction via $S_i \leftarrow S_i / \exp \left[C\sigma^2_i/2\right]$, obtaining unbiased attention scores.

In summary, we can propagate and track the effective error through the layers: given the natural parameter of $X^{(\ell-1)}$, the variance can be estimated using analytic expressions, which then can be used to correct the attention scores.

\textbf{Privacy guarantee.} Unlike the privacy proof of Phantom Clipping, the input-output behavior here differs from that of DP-SGD. However, note that we only change the forward propagation procedure of the model, and such a modification does not create a dependency between different samples in a minibatch. Therefore, DP-SGD still bounds the sensitivity by clipping the gradient norm.

\textbf{Discussion.} In the above, we show how to efficiently approximate the variance term due to the intractability of obtaining the exact quantity, primarily using techniques from Bayesian deep learning. The key here is that we only need to use these techniques in a black-box manner. They can be replaced by other, more sophisticated techniques to achieve the same goal.

\subsection{Empirical evaluation}
We empirically evaluate our Re-Attention Mechanism on two public recommendation datasets collected from real-world scenarios: MovieLens~\citep{harper2015movielens} and Amazon~\citep{mcauley2015image}. 
The task is to predict the text item given a sequence of items as input, i.e., the most canonical use case of Transformer models.
\Cref{fig:convergence} shows the model accuracy every five epochs during training. Experimental details and more results are in \Cref{apdx:empiricalreattn}. 
Two points are worth mentioning here.
\begin{itemize}
    \item Overall, the Re-Attention Mechanism consistently improves the training stability, which renders the model notably more stable and enables better convergence during differentially private training. In contrast, the vanilla Transformer model suffers from high variance and/or substantial fluctuation, especially on Amazon.
    \item More quantitatively, our theoretical analysis implies that the degree of attention distraction is related to the degree of the imbalance in the data distribution and the privacy level of DP.
    This is exemplified by the Re-Attention Mechanism's enhanced effectiveness on Amazon (a more imbalanced dataset compared to MovieLens) and at a privacy setting of $\varepsilon=5$, (greater privacy than $\varepsilon=10$). Recall that our theoretical analysis 
    in \Cref{sec:attentiondistraction} shows that the more challenging the task, the more severely the model will suffer from the attention distraction phenomenon, and in turn, the greater advantage will be enjoyed by the Re-Attention Mechanism. The empirical results are in alignment with our theory.
\end{itemize}

\begin{figure}[ht!]
\centering

\begin{subfigure}{0.23\textwidth}
\centering
\setlength{\abovecaptionskip}{0.cm}
\includegraphics[width=4cm]{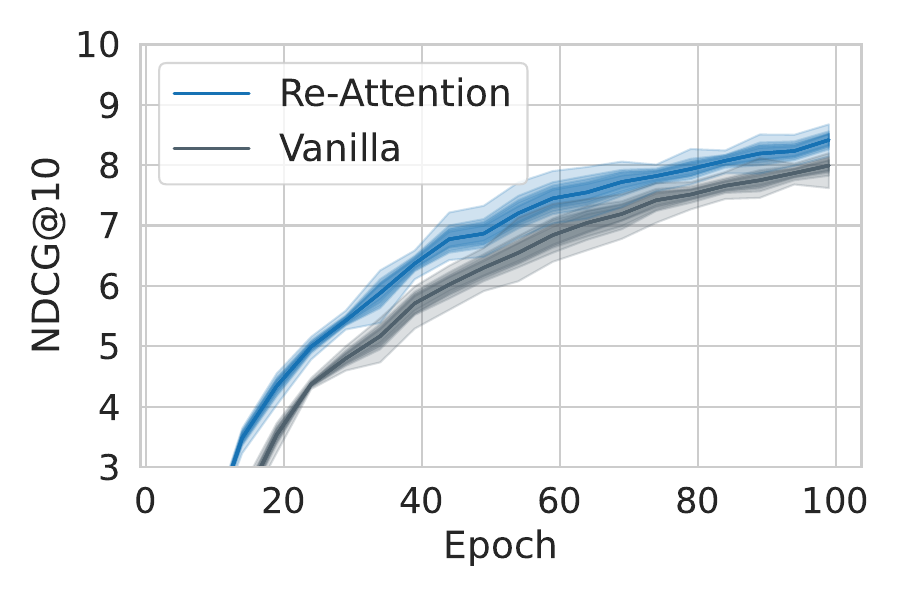}
\caption{MovieLens ($\varepsilon=10$)}
\label{fig:ndcg_ml1m_10}
\end{subfigure}
\begin{subfigure}{0.23\textwidth}
\centering
\setlength{\abovecaptionskip}{0.cm}
\includegraphics[width=4cm]{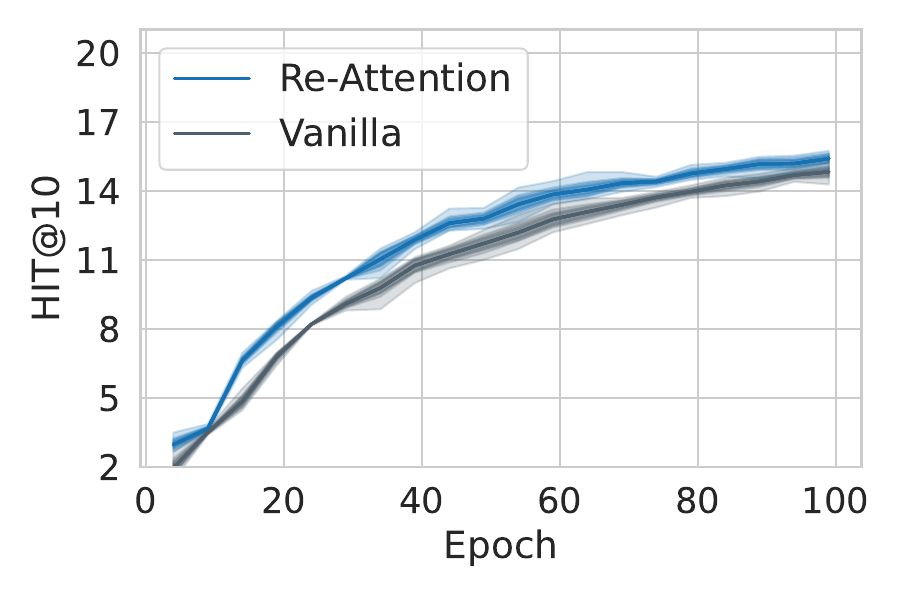}
\caption{MovieLens ($\varepsilon=10$)}
\label{fig:hit_ml1m_10}
\end{subfigure}

\begin{subfigure}{0.23\textwidth}
\centering
\setlength{\abovecaptionskip}{0.cm}
\includegraphics[width=4cm]{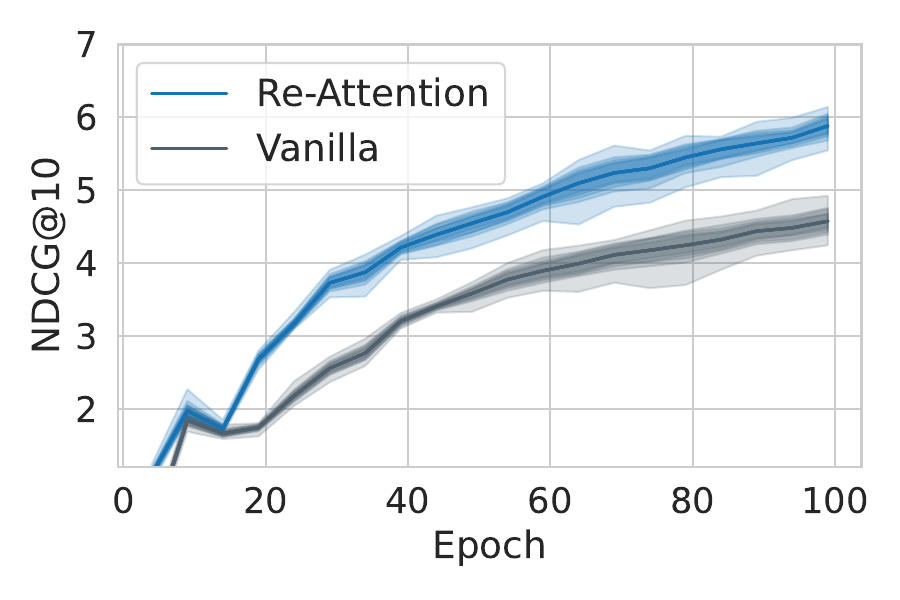}
\caption{MovieLens ($\varepsilon=5$)}
\label{fig:ndcg_amazon}
\end{subfigure}
\begin{subfigure}{0.23\textwidth}
\centering
\setlength{\abovecaptionskip}{0.cm}
\includegraphics[width=4cm]{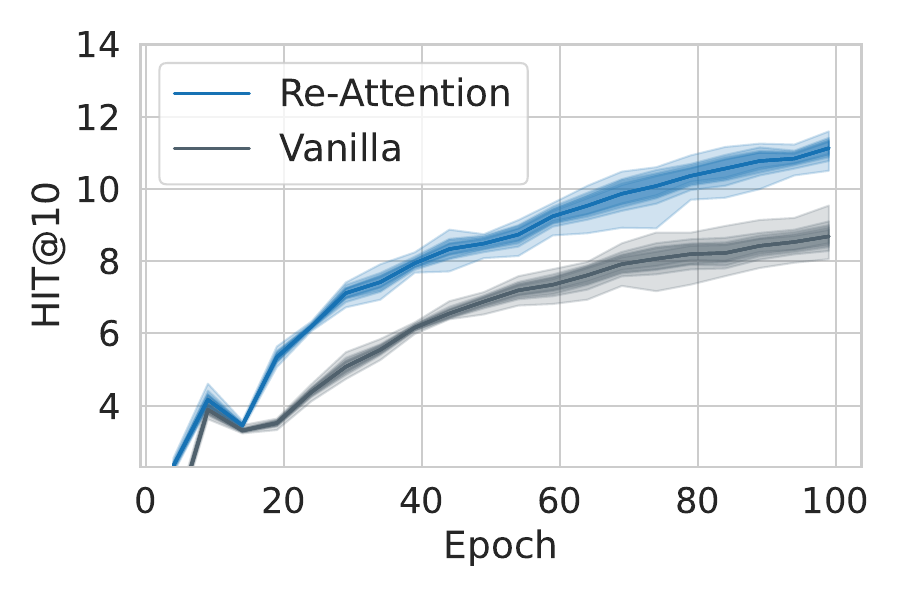}
\caption{MovieLens ($\varepsilon=5$)}
\label{fig:hit_ml1m_5}
\end{subfigure}

\begin{subfigure}{0.23\textwidth}
\centering
\setlength{\abovecaptionskip}{0.cm}
\includegraphics[width=4cm]{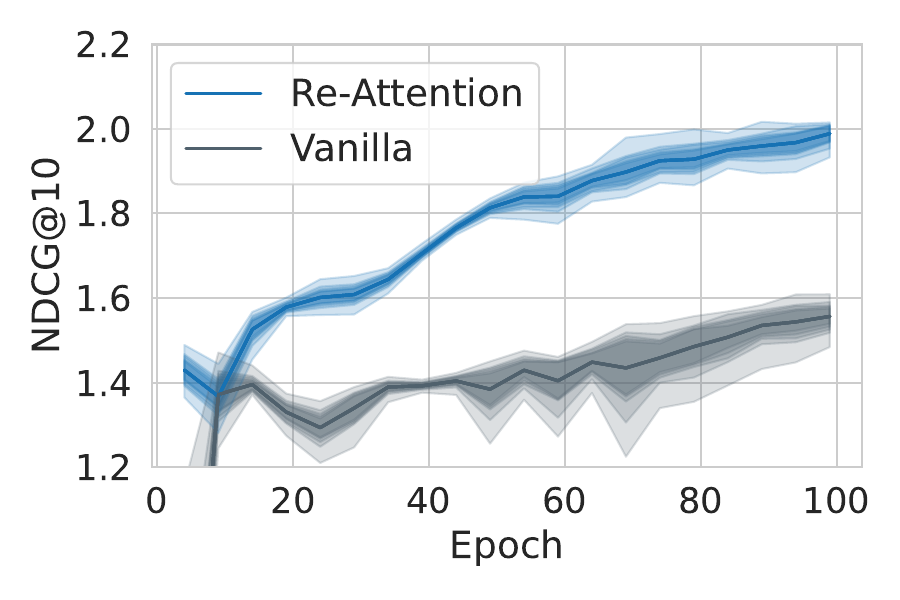}
\caption{Amazon ($\varepsilon=10$)}
\label{fig:ndcg_amazon_10}
\end{subfigure}
\begin{subfigure}{0.23\textwidth}
\centering
\setlength{\abovecaptionskip}{0.cm}
\includegraphics[width=4cm]{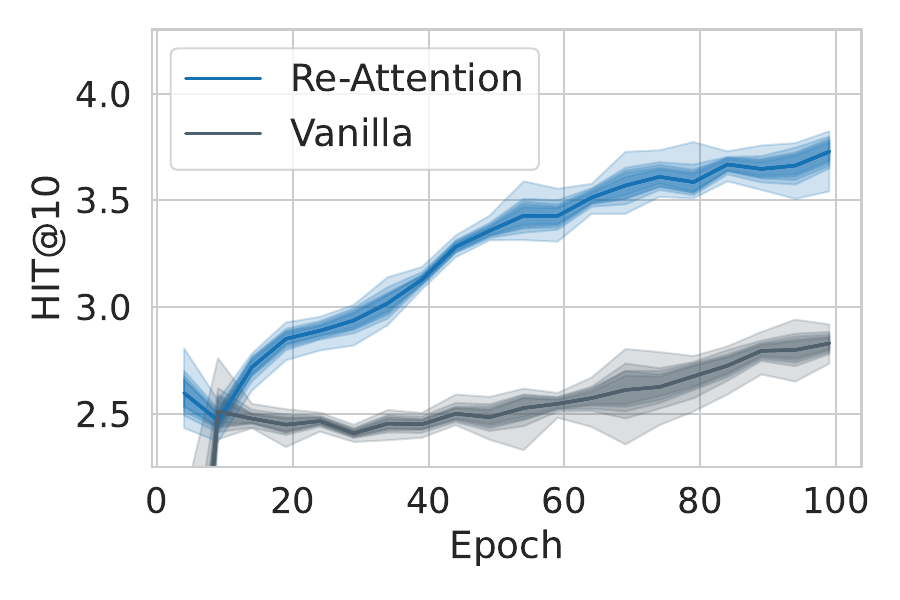}
\caption{Amazon ($\varepsilon=10$)}
\label{fig:hit_amazon_10}
\end{subfigure}

\begin{subfigure}{0.23\textwidth}
\centering
\setlength{\abovecaptionskip}{0.cm}
\includegraphics[width=4cm]{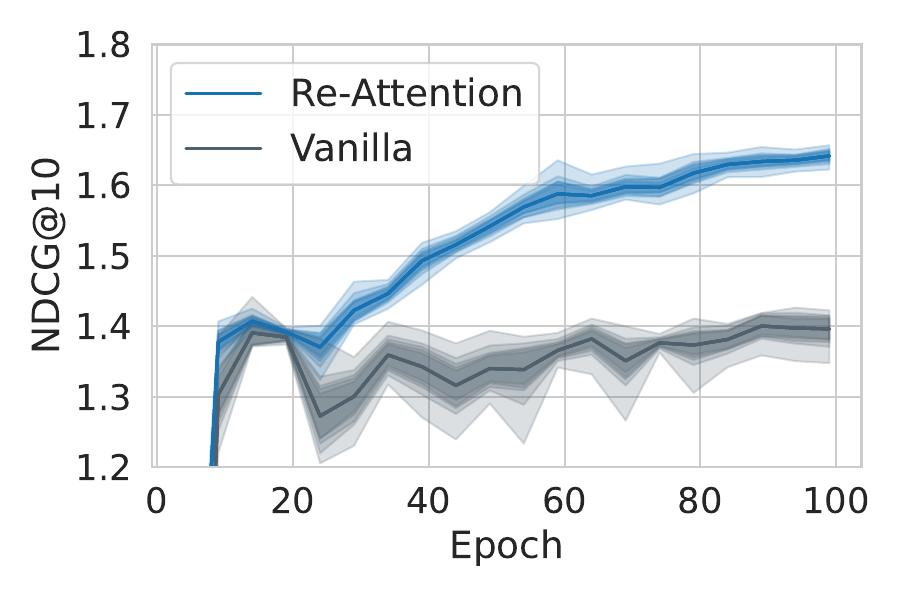}
\caption{Amazon ($\varepsilon=5$)}
\label{fig:ndcg_amazon_5}
\end{subfigure}
\begin{subfigure}{0.23\textwidth}
\centering
\setlength{\abovecaptionskip}{0.cm}
\includegraphics[width=4cm]{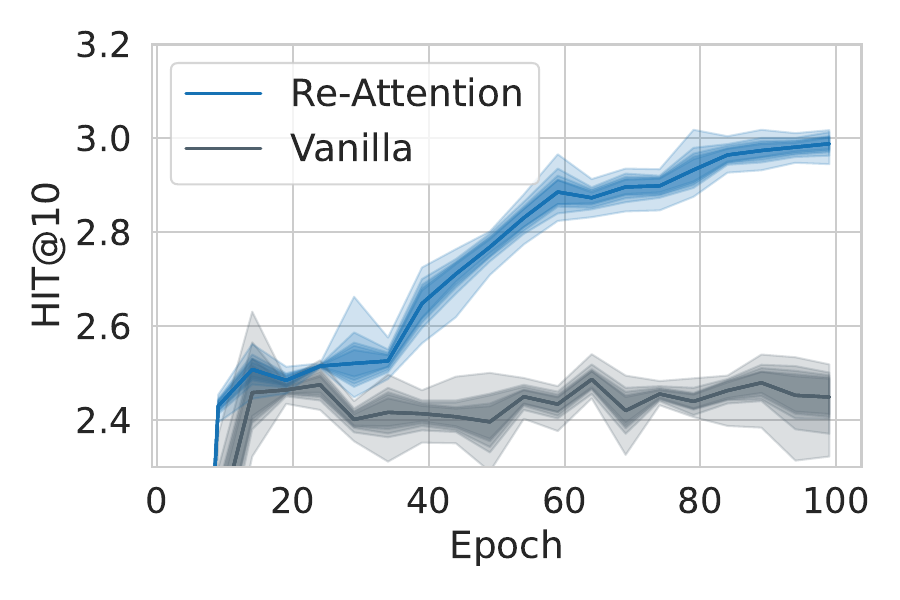}
\caption{Amazon ($\varepsilon=5$)}
\label{fig:hit_amazon_5}
\end{subfigure}

\caption{Each run is repeated five times with independent random seeds, with test accuracy (i.e., NDCG@10(\%) and HIT@10(\%)) reported every five epochs. The graduated shading (best viewed zoomed in) represents confidence intervals from 60\% to 100\%. }
\label{fig:convergence}
\end{figure}
\section{Conclusion}
In this paper, we systematically study the problem of training differentially private Transformer models (Problem A) under a modular treatment. In particular, our main focus is on its reduction to the more basic problem of training vanilla DP Transformers (Problem B).

We first identify the gap between Problem A and Problem B: 1) The attention distraction phenomenon, which negatively affects the intended functionality of that neural net module, and 2) Lack of compatibility with existing techniques for efficient gradient clipping.

To bridge the gap and achieve the reduction, we correspondingly propose 1) Re-Attention Mechanism, which corrects the attention distraction and leads to unbiased attention scores. Thus it restores the intended functionality of self-attention. Now each module of the neural nets is supposed to operate properly, just as DP vanilla neural nets.
2) Phantom Clipping, which provides support for embedding sharing, enabling the Transformer model to benefit from efficient DP-SGD gradient clipping, just as DP vanilla neural nets.



\textbf{Limitation and  open questions.}
The main open question raised by this work is, under our modular treatment, whether our `reduction' is complete or not, in the sense that there might be other subtle yet unrecognized  hardness underlying  training Transformer with differential privacy. We are unable to provide a definitive answer to this question. 


On the empirical side, we only use small Transformer models for empirical evaluation. While the use of small model is justified considering the computing resource of end devices, making it preferable for privacy protection with local inference. 
Specifically, it eliminates the need for uploading sensitive data to the cloud service for model inference, which may become a privacy concern. 
It would be interesting to scale our method to large models.

Finally, it would be valuable to explore other specific deep learning models under differential privacy in a modular approach as we suggest. Potentially, the techniques from Bayesian learning may also be useful there.

\section*{Acknowledgements}
We thank the reviewers for  giving useful comments. 
We thank Gautam Kamath for encouraging us to explore differentially private deep learning.
We thank  David Evans and Yaodong Yu for helpful comments and  discussion. 

\section*{Impact Statement}
This paper presents work whose goal is to enhance personal privacy protection.




\newpage
\appendix
\onecolumn

\section{Related Work}
\label{apdx:related}
\subsection{Differential Private Deep Learning}
\citet{papernot2021tempered} suggested tempered sigmoid activations to control the gradient norm explicitly, and in turn support faster convergence in the settings of differentially private
ML. 
\citet{MohapatraSH0022} studied the intrinsic connection between the learning rate and clipping norm hyperparameters and show that adaptive optimizers like DPAdam enjoy a significant advantage in the process of honest hyperparameter tuning.
\citet{wei2022dpis} employed importance sampling (IS) in each SGD iteration for minibatch selection. 

One line of work focuses on adaptive optimization of differentially private machine learning. \citet{asi2021private} proposed adaptive stepsizes as a variant of SGD. \citet{li2022private} uses non-sensitive side information to
precondition the gradients, allowing the effective use of adaptive methods in private settings.
 
Another line of work studies the loss landscape of DP-SGD in comparison to SGD.
\citet{wang2021dplis} first highlighted the problem of DP-SGD being stuck in local minima due to the training instability. They constructed a smooth loss function that favors noise-resilient models lying in large flat regions of the loss landscape. 
\citet{shamsabadi2021losing} proposed that loss functions with smaller norm can reduce the impact of clipping and thus create a smoother loss function. \citet{park2023differentially} made use of sharpnessaware training without additional privacy costs.

The most related work is \citet{li2022large}, which finetunes large language models with differential privacy. They propose Ghost Clipping, which is a technique that enables efficient per-sample gradient clipping without instantiating per-sample gradient. Our Phantom Clipping can be viewed as an extension of Ghost Clipping that additionally handles parameter sharing of the embedding layer. They also introduce the idea of \emph{effective noise multiplier} in order to explain the role of batch size in private learning. Our \emph{effective error} (\Cref{eq:effectiveerror}) can be viewed as its generalization in order to account for the inherent input sparsity (i.e., only a small portion of tokens appear in one training sequence). \citet{yu2023vip} proposed to train a Vision Transformer model with differential privacy from scratch.

Another related work includes \citet{anil-etal-2022-large}, which establishes a baseline for BERT-Large pretraining with DP. They introduce several strategies to help private training, such as large weight decay and increasing batch size schedule. Notably, these strategies are independent yet complementary to the methodologies utilized in this work, thereby offering potential avenues for an integrated approach. 


The line of work~\cite{feyisetan2020privacy,mattern-etal-2022-limits,carvalho2023tem,utpala2023locally} focuses on privatization by pre-processing on the training texts and then using non-private training while our work focuses on private training with DP-SGD. Results are incomparable because they rely on a non-standard and generalization notion of privacy to preserve utility. The work~\cite{DPforward} proposes to differentially private fine-tuning Language Models in the forward pass, which can also be  regarded as privatization by pre-processing, but on the generalized data (i.e., features extracted by neural nets). The work~\cite{mai2024splitanddenoise} studies private inference of Language Models with local differential privacy.

Previous work has also considered the differentially private learning algorithms on heavy-tailed data~\cite{wang2020differentially,hu2022high,pmlr-v162-kamath22a}. This line of research is mainly concerned with differential private stochastic optimization (DP-SCO). Note that the notion of heavy-tailed there is different from the focus of this work. As pointed out in~\citet{pmlr-v162-kamath22a}, the setting they actually consider is dealing with heavy-tailed gradients due to unbounded values in the input data. 


\section{Phantom Clipping}
By unrolling the expression, ~\Cref{alg:pc} directly follows from the following claim.
\label{apdx:pfphan}
\begin{claim}
    (\textbf{Phantom Clipping})  For $1\leq i \leq B$, the norm of the per-sample gradient $\|g_{i, E}\|$ with respect to the shared embedding layer $E$ can be efficiently evaluated, without instantiating $g_{i, E}$, by
\begin{equation}
    \|g_{E}\|_i = \left( \langle (a_{\mathcal{S}})_i \cdot (a_{\mathcal{S}})_i^T, \nabla (e_{\mathcal{S}})_i \cdot \nabla (e_{\mathcal{S}})_i^T \rangle^2 + \|(\nabla e_{\mathcal{C}})_i\|^2 + 2 \cdot \langle \nabla (e_{\mathcal{S}})_i, (a_{\mathcal{S}})_i \cdot(\nabla e_{\mathcal{C}})_i \rangle  \right)^{\frac{1}{2}},
    \label{eq:phantom2}
\end{equation}
where $\langle\cdot, \cdot \rangle$ is the inner product of two matrices being of the same shape.
\end{claim}

\begin{proof}
For simplicity, we will omit the per-sample index $i$ throughout this proof and simply let batch size $B = 1$ without loss of generality.
From the chain rule, the per-sample gradient with respect to the embedding layer $E$ is
\begin{equation}
\begin{split}
    g_{E} &= 
    \frac{\partial \mathcal{L}}{\partial e_{\mathcal{S}}}\cdot \frac{\partial e_{\mathcal{S}}}{\partial E} + 
    \frac{\partial \mathcal{L}}{\partial e_{\mathcal{C}}}\cdot \frac{\partial e_{\mathcal{C}}}{\partial E}\\
    &=\underbrace{a_{\mathcal{S}}^T \cdot \nabla e_{\mathcal{S}}}_{g_{E}^{(1)}} + 
    \underbrace{a_{\mathcal{C}}^T \cdot \nabla e_{\mathcal{C}}}_{g_{E}^{(2)}},
\end{split}
\end{equation}
where $a_{s} \in \{0, 1\}^{L\times M}$ (or $a_{c} \in \{0, 1\}^{M\times M}$) is the one-hot encodings of the input sequence $s_i$ (or those of the candidate tokens for the output probability) in a minibatch, and $e_{s} \in \mathbb{R}^{L\times d}$ (or $e_{c} \in \mathbb{R}^{M\times d}$) be output of the (shared) embedding layer $E$ when fed into $a_{s}$ (or $a_{c}$).
Denote the first segment of the right-hand side (RHS) as $g_{E}^{(1)}$, the second segment of the RHS as $g_{E}^{(2)}$. Then we have
\begin{equation}
     \left\|g_{E}\right\|^2 = \left\|g_{E}^{(1)} + g_{E}^{(2)}\right\|^2 =  
     \left\|g_{E}^{(1)}\right\|^2 + \left\| g_{E}^{(2)}\right\|^2
     + 2\cdot \langle g_{E}^{(1)}, g_{E}^{(2)}\rangle.
     \label{eq:breakdown}
\end{equation}
Ghost Clipping~\citep{li2022large} allows us to evaluate $\left\|g_{ E}^{(1)}\right\|_F$  without instantiating $g_{E}^{(1)}$, the formula is given by
\begin{equation}
    \begin{split}
        \left\|g_{E}^{(1)}\right\| = \langle a_{\mathcal{S}}a_{\mathcal{S}}^T, \nabla e_{\mathcal{S}} \nabla e_{\mathcal{S}}^T \rangle.
    \end{split}
    \label{eq:phantom_p_1}
\end{equation}
Likewise, we have
\begin{equation}
    \begin{split}
        \left\|g_{E}^{(2)}\right\| = \langle a_{\mathcal{C}}a_{\mathcal{C}}^T, \nabla e_{\mathcal{C}} \nabla e_{\mathcal{C}}^T \rangle.
    \end{split}
    \label{eq:ghost_emb}
\end{equation}
With appropriate implementation, $a_{\mathcal{C}}$ is the one-hot encoding of $[1, 2, 3, ..., M]$, thus $a_{\mathcal{C}}$ is an identity matrix and \Cref{eq:ghost_emb} can be further simplified as
\begin{equation}
    \begin{split}
        \left\|g_{E}^{(2)}\right\|= \langle \mathbf{I}, \nabla e_{\mathcal{C}} \nabla e_{\mathcal{C}}^T \rangle = \sum_{j,k} (\nabla e_{\mathcal{C}})_{j,k}\cdot (\nabla e_{\mathcal{C}})_{j,k}  =  \|\nabla e_{\mathcal{C}}\|^2.
    \end{split}
    \label{eq:phantom_p_2}
\end{equation}
Note that this simplification saves us the memory footprint of $O(B M'^2)$ for evaluating $\nabla e_{\mathcal{C}} \nabla e_{\mathcal{C}}^T$ in \Cref{eq:ghost_emb}.

Therefore, computing the gradient norm of shared embedding reduces to computing $\langle g_{E}^{(1)}, g_{E}^{(2)}\rangle$ in \Cref{eq:breakdown},
\begin{equation}
\begin{split}
    \langle g_{E}^{(1)}, g_{E}^{(2)}\rangle 
    &= \langle a_{\mathcal{S}}^T \cdot  \nabla e_{\mathcal{S}}, a_{\mathcal{C}}^T \cdot \nabla e_{\mathcal{C}}\rangle \\
    &= \sum_{j=1}^{M}\sum_{k=1}^{d}\left(\sum_{i=1}^{L} (a_{\mathcal{S}})_{ij} \cdot (\nabla e_{\mathcal{S}})_{ik}\right)\left(\sum_{i=1}^{M'} (a_{\mathcal{C}})_{ij} \cdot (\nabla e_{\mathcal{C}})_{ik}\right)\\
    &= \sum_{i_1=1}^{L}\sum_{i_2=1}^{M'} \left(\sum_{j=1}^{M}\sum_{k=1}^{d} (a_{\mathcal{S}})_{i_1j} \cdot (\nabla e_{\mathcal{S}})_{i_1k} \cdot (a_{\mathcal{C}})_{i_2j} \cdot (\nabla e_{\mathcal{C}})_{i_2k} \right) \\
    &= \sum_{i_1=1}^{L}\sum_{i_2=1}^{M'}\left(\sum_{j=1}^{M} (a_{\mathcal{S}})_{i_1j} \cdot (a_{\mathcal{C}})_{i_2j}\right)\left(\sum_{k=1}^{d} (\nabla e_{\mathcal{S}})_{i_1k} \cdot (\nabla e_{\mathcal{C}})_{i_2k}\right)\\
    &= \sum_{i_1=1}^{L}\sum_{i_2=1}^{M'}\langle (a_{\mathcal{S}})_{i_1}, (a_{\mathcal{C}})_{i_2}\rangle \cdot \langle (\nabla e_{\mathcal{S}})_{i_1}, (\nabla e_{\mathcal{C}})_{i_2}\rangle\\
    &= \sum_{i_1=1}^{L}\sum_{i_2=1}^{M'}\mathbb{I}[i_2 = \operatorname{onehot}^{-1}((a_{\mathcal{S}})_{i_1})] \cdot \langle (\nabla e_{\mathcal{S}})_{i_1}, (\nabla e_{\mathcal{C}})_{i_2}\rangle\\
    &= \sum_{i_1=1}^{L}\langle (\nabla e_{\mathcal{S}})_{i_1}, (a_\mathcal{S})_{i_1} \cdot \nabla  e_{\mathcal{C}}\rangle\\
    &= \langle (\nabla e_{\mathcal{S}}), a_\mathcal{S}\cdot (\nabla e_{\mathcal{C}})\rangle.\\
\end{split}
\label{eq:phantom_p_3}
\end{equation}
Combining \Cref{eq:phantom_p_1}, \Cref{eq:phantom_p_2} and ~\Cref{eq:phantom_p_3} yields the conclusion.
\end{proof}

\section{Re-Attention Mechanism}

\subsection{Derivation of \Cref{cl:errorinst2}}
\label{apdx:errorinst}

Taking the average over the minibatch in \Cref{eq:DPSGD} can be considered noise reduction of $O(1/B)$. Fix the noise multiplier $\sigma_{dp}$. As the size of the batch increases, the amount of DP noise incorporated into the parameter decreases correspondingly. Suppose that token $i$ is absent from the current training sequence. Its input embedding will not be activated and thus will not be properly trained in this iteration, but the DP noise will be nevertheless injected into its embedding. The concept of effective error is introduced to account for this phenomenon.

\begin{proof}
    It reduces to derive the formula for effective batch size $B_{\eff}^{\theta}$ in \Cref{eq:effectiveerror}.
    Recall that its definition is given by
    \begin{equation}
        B_{\eff}^{\theta} = \E_{\mathcal{B}\stackrel{\text{\tiny{i.i.d.}}}{~\sim~}\mathcal{D}^B} \left[\sum_{i=1}^B \mathbb{I}\left[ R_{\theta}(\mathcal{B}_i)\right]\right].
    \end{equation}
    For each layer parameterized by $W$ within the Transformer block, its effective batch size is $B_{\eff}^{W} = B$, since $R_{W}(\mathcal{B}_i) = 1$.

    For the embedding layer $E$, its effective batch size is
    \begin{equation}
    \begin{split}
         B_{\eff}^{E_i} &=  \E_{\mathcal{B}\stackrel{\text{\tiny{i.i.d.}}}{~\sim~}\mathcal{D}^B} \left[\sum_{j=1}^B \mathbb{I}\left[ R_{E_i}(\mathcal{B}_j)\right]\right]\\
         &= \sum_{j=1}^B  \E_{\mathcal{B}_j\stackrel{\text{\tiny{i.i.d.}}}{~\sim~}\mathcal{D}} \left[ \mathbb{I}\left[ R_{E_i}(\mathcal{B}_j)\right]\right] ~~~~\text{(Linearity of Expectation)}\\
         &= \sum_{j=1}^B  \E_{\mathcal{B}_j\stackrel{\text{\tiny{i.i.d.}}}{~\sim~}\mathcal{D}} \left[ \mathbb{I}\left[ {\rm token}~i\in \mathcal{B}_j \right]\right]\\
         &= \sum_{j=1}^B  p_i\\
         &= B\cdot p_i,\\
    \end{split}
    \end{equation}
    where $p_i$ is the frequency of token $i$ (i.e., the probability of token $i$'s occurrence in data). 
\end{proof}

\subsection{Propagation of Natural Parameters}
\label{apdx:npnprop}
For linear transformation, $X^{(l)}=X^{(l-1)}W$, we can propagate the variance as
\begin{equation}
    \sigma_{X^{(l)}}^2 = \sigma_{X^{(l-1)}}^2 \cdot \sigma_W^2 + \sigma_{X^{(l-1)}}^2 \cdot \mu_W^2 +  \sigma_W^2 \cdot (\mu_{X^{(l-1)}})^2
    \label{eq:linearnpn}.
\end{equation}
For nonlinear activation functions, e.g., $X^{(l)}=\operatorname{ReLU}\left(X^{(l-1)}\right)$, we can propagate the variance as
\begin{equation}
\begin{split}
    \sigma_{X^{(l)}}^2 = \Phi(\frac{c}{\sqrt{d}})(c^2+d) + \frac{c\sqrt{d}}{\sqrt{2\pi}}\exp(-\frac{1}{2}\frac{c^2}{d}) - c^2,
\end{split}    
\label{eq:relunpn}
\end{equation}
where $\Phi(\cdot)$ is the cumulative density function (CDF) of the standard Gaussian distribution, $c$ and $d$ are the natural parameter of  $X^{(l-1)}$.

\label{apdx:npnproof}
\begin{lemma}
    Let $X$, $Y$ be two independent random variables, $Z=XY$, then the variance of $Z$ can be expressed as
    \begin{equation}
        \begin{split}
            \var[Z] = \var[XY] =  \E[X^2]\E[Y^2] - \E[XY]^2.
        \end{split}
    \end{equation}
    \label{le:prod_var}
\end{lemma}
\Cref{le:prod_var} directly implies \Cref{eq:linearnpn} as follows.
\begin{proof}
    Suppose the linear transformation is given by $X^{(l)}=X^{(l-1)}W$, then we have
    \begin{equation}
    \begin{split}
        \var[X^{(l)}] &= \E[(X^{(l-1)})^2]\E[W^2] - \E[X^{(l-1)}W]^2\\
        &= (\E[X^{(l-1)}]^2 + \var[X^{(l-1)}])(\E[W]^2 + \var[W])- \E[X^{(l-1)}]^2\E[W]^2\\
        &= \var[X^{(l-1)}]\var[W] + \E[X^{(l-1)}]^2\var[W] + \E[W]^2\var[X^{(l-1)}].\\
    \end{split}
    \end{equation}
\end{proof}

\begin{lemma}
\label{le:max_var}
    Let $X_1$, $X_2$ be two independent Gaussian random variables, where $X_i\sim \mathcal{N}(\mu_i, \sigma_i),i=1,2$. Let $Z=\operatorname{max}(X_1, X_2)$.
    \begin{equation}
        \begin{split}
            &\E[Z] = \mu_1 \Phi(\gamma) + \mu_2\Phi(-\gamma) + \nu\phi(\gamma)\\
            &\E[Z^2] = (\mu_1^2+\sigma_1^2) \Phi(\gamma) + (\mu_2^2 + \sigma_2^2)\Phi(-\gamma) + (\mu_1 + \mu_2)\nu\phi(\gamma),
        \end{split}
        \label{eq:max_var}
    \end{equation}
    where $\Phi(\cdot)$ is the cumulative density function (CDF) of the standard Gaussian distribution, $\phi(\cdot)$ is the probability density function (PDF) of the standard Gaussian distribution, $\nu$ = $\sqrt{\sigma_1^2 + \sigma_2^2}$, and $\gamma=(\mu_1-\mu_2)/\nu$.
\end{lemma}
\Cref{le:max_var} directly implies \Cref{eq:nonprop} as follows.
\begin{proof}
    Let $X^{(l)}=\operatorname{ReLU}\left(X^{(l-1)}\right)$ be the ReLU activation. Substitute $\mu_1=\E[X^{(l-1)}], \sigma_2 =\var[X^{(l-1)}], \mu_2=\sigma_2 = 0$ into \Cref{eq:max_var}. Leveraging $\var[X^{(l)}]=\E[(X^{(l)})^2] -\E[X^{(l)}]^2$  yields the conclusion.
\end{proof}
\begin{remark}
    \textbf{GELU activation.} GELU function is another widely used activation function within Transformer models. GELU can be viewed as a smooth version of ReLU (see \Cref{apdx:fig:gelu}), where their forward propagation is similar, and the major distinction lies in the numerical behavior of backpropagation. Since error propagation is only concerned with forward propagation behavior, we can also use \Cref{eq:nonprop} to approximate the variance of GELU output. \Cref{apdx:tb:analytic} shows the analytic error propagation for ReLU and GELU activation, compared with the sampling-based results.
\end{remark}
\begin{figure}[ht]
\vspace{-0.5cm}
\centering
\setlength{\abovecaptionskip}{0.cm}
\includegraphics[width=8cm]{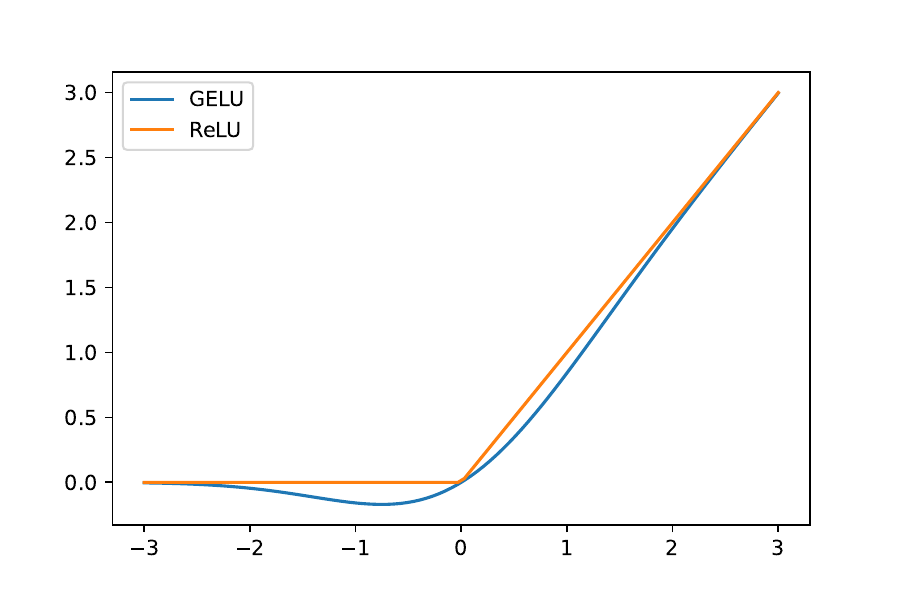}

\caption{GELU activation and ReLU activation.}
\label{apdx:fig:gelu}

\end{figure}

\begin{table}[ht]
\scriptsize
\centering
\caption{Analytic error propagation for ReLU and GELU activation.}
\begin{tabular}{@{}c|c|cccccc|c@{}}
\toprule
\multirow{2}{*}{Input}      & \multirow{2}{*}{Activation} & \multicolumn{6}{c|}{Sampling-based}                       & \multirow{2}{*}{Analytic} \\
                            &                             & 10      & 100     & 1000    & 10000   & 100000  & 1000000 &                           \\ \midrule\midrule
\multirow{2}{*}{$\mathcal{N}(0, 0.01)$} & \textsc{ReLU}                        & 4.08e-6 & 3.60e-5 & 3.93e-5 & 3.45e-5 & 3.40e-5 & 3.40e-5 & \multirow{2}{*}{\textbf{3.40e-5}}  \\
                            & \textsc{GELU}                         & 2.48e-5 & 2.69e-5 & 2.72e-5 & 2.57e-5 & 2.50e-0 & 2.49e-5 &                           \\ \midrule
\multirow{2}{*}{$\mathcal{N}(0, 0.1)$}  & \textsc{ReLU}                        & 0.0030  & 0.0031  & 0.0037  & 0.0034  & 0.0035  & 0.0034  & \multirow{2}{*}{\textbf{0.0034}}   \\
                            & \textsc{GELU}                        & 0.0030  & 0.0025  & 0.0027  & 0.0025  & 0.0026  & 0.0025  &                           \\ \midrule
\multirow{2}{*}{$\mathcal{N}(0, 1)$}    & \textsc{ReLU}                        & 0.5299  & 0.2361  & 0.3649  & 0.3451  & 0.3387  & 0.3418  & \multirow{2}{*}{\textbf{0.3408}}   \\
 & \textsc{GELU}  & 0.5525  & 0.2306  & 0.3719  & 0.3506  & 0.3433  & 0.3467  &  \\ \bottomrule
\end{tabular}
\label{apdx:tb:analytic}
\end{table}




\section{Empirical Evaluation}
\label{apdx:exp}
\subsection{Datasets}
\label{apdx:exp:dataset}
\begin{figure}[ht]
\centering
\setlength{\abovecaptionskip}{0.cm}
\includegraphics[width=7cm]{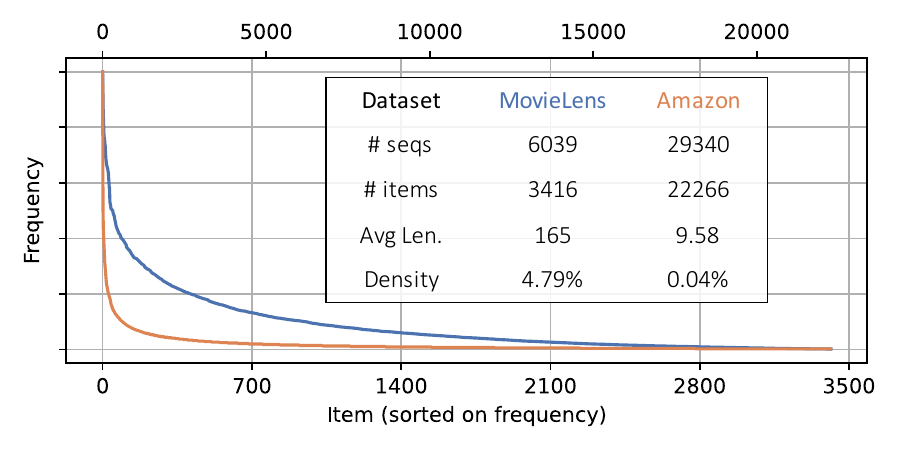}
\caption{Data in the real-world scenarios exhibits long-tailed (also known as, power-law) distribution.
}
\vspace{-0.2cm}
\label{fig:dataset}
\end{figure}
We conduct experiments on two public recommendation datasets collected from real-world scenarios: MovieLens~\citep{harper2015movielens} and Amazon~\citep{mcauley2015image}. \Cref{fig:dataset} shows their data distributions, illustrating the prevalence of long-tailed distributions, where a small number of items are extremely popular and have relatively high frequency while other items occur infrequently. The embedded table above the `long tail' reports the statistics of the two datasets, showing that the two datasets vary significantly in size and sparsity.

\Cref{fig:dataset} shows their data distributions, illustrating the prevalence of long-tailed distributions, where a small number of items are extremely popular and have relatively high frequency while other items occur infrequently. The embedded table above the `long tail' reports the statistics of the two datasets, showing that the two datasets vary significantly in size and sparsity.

\textbf{MovieLens.} The MovieLens dataset~\citep{harper2015movielens} is often used in the development and evaluation of collaborative filtering algorithms, which are used to make personalized recommendations based on user behavior. It is a benchmark dataset in the field of recommender systems due to its size, longevity, and richness of user-item interactions. We use the version (MovieLens-1M) that includes 1 million user behaviors.

\textbf{Amazon.} A series of datasets introduced in~\citep{mcauley2015image}, comprising large corpora of product reviews crawled from
Amazon.com. Top-level product categories on Amazon are
treated as separate datasets. We consider the ‘Games.’ category. This dataset is notable for its high sparsity and variability.

We follow \citep{kang2018self} for the data preprocessing. We use timestamps to determine the sequence order of
actions. Each user is associated with a training sequence (i.e., his chronological behavior). We discard users and items with fewer than five related actions. For data partitioning, the last token of each sequence is left for testing.

It is worth noting that since each user is exclusively associated with exactly one training sample (sequence) in the training data, the DP guarantee we provide is user-level. That is, removing all information pertaining to a specific user  yields an indistinguishable model.

\subsection{Model Architecture}
\label{apdx:exp:preprocessing}
We use the standard Transformer encoder described in~\citep{vaswani2017attention}. The model dimension is set to 64. The number of heads in the Attention Mechanism is set to 1. The number of Transformer blocks is set to 2. Our model adopts a learned (instead of fixed) positional embedding. The model size is similar to that in \cite{ramaswamy2020training}, which is suitable for deployment on the user devices.

\subsection{Hyperparameters}
The number of epochs is set to 100. The batch size is chosen from $\{256, 512, 1024, 2048, 4096\}$. The learning rate is chosen from $\{10^{-3}, 3\times10^{-3}, 5\times10^{-3},  7\times10^{-3},  9\times10^{-3}\}$. The dropout rate is 0.2 for MovieLens and 0.5 for Amazon (due to its high sparsity). We use the Adam optimizer with a weight decay of $10^{-5}$.

\subsection{Evaluation Metrics}
HIT@$k$ measures whether the relevant (i.e., ground truth) item is present within the top-$k$ items in prediction list. It is a binary metric indicating the presence (hit) or absence of relevant items.
\begin{equation}
\text{HIT@$k$} = \begin{cases} 
1, & \text{if the next item is in top-$k$ prediction} \\
0, & \text{otherwise} 
\end{cases}
\end{equation}

NDCG@$k$ measures the performance of a recommendation system based on the graded relevance of the recommended items. It is normalized based on the ideal order of items.
\begin{equation}
    \text{NDCG@$k$} = \frac{\text{DCG@$k$}}{\text{IDCG@$k$}}
\end{equation}
where
\begin{equation}
\text{DCG@$k$} = \sum_{i=1}^{k} \frac{2^{rel_i} - 1}{\log_2(i+1)}
\end{equation}
Here, $rel_i$ is the relevance score of the item at position $i$ in the recommendation list. Namely, if the $i$-th item (sorted by prediction score of the model) is equal to the ground truth, $rel_i=1$ otherwise 0.
IDCG@$k$ (Ideal Discounted Cumulative Gain at $k$) is the maximum possible 
DCG@$k$, obtained by placing the most relevant items in the top positions (i.e., the item with highest prediction score is equal to the ground truth). It is calculated similarly to DCG@$k$ but for the ideal order.

We adhere to the evaluation method advocated in~\cite{krichene2020sampled}, i.e., ranking all the items rather than adopting the sampled metrics where only a smaller set of random items and the relevant items are ranked. Note that under this evaluation method the accuracy value is greatly lower than that obtained by sampled metrics, but is more consistent and meaningful.

\subsection{Empirical Evaluation of Phantom Clipping}
\label{apdx:phanempirical}
To show the importance of parameter sharing when training Transformer models with DP-SGD, we conduct experiments under the following three settings: (1) parameter sharing of the embedding layer, which aligns with the standard treatment in Transformer; (2) no parameter sharing; and (3) no parameter sharing coupled with a reduced embedding dimension by half.
Note that the third setting is included to account for the potential impact of model dimension on accuracy in private training, given the difference in the number of parameters between models with and without parameter sharing. Model performance across different hyperparameters is shown in \Cref{fig:embsharing}. The consistency and significance of the performance improvement brought by parameter sharing during private training are not hard to perceive. The essence of embedding sharing lies in the assumption that, by tying the embedding of the input and output layers, the representation of each token remains consistent throughout its retrieval. This inductive bias enhances the statistical efficiency of the model, enabling improved generalization. When training with DP-SGD on limited training data, the model must independently uncover this relationship from the noisy gradients with a low signal-to-noise ratio, heightening the convergence challenge.

\begin{figure}[ht]

\centering
\begin{subfigure}{0.3\textwidth}
\centering
\includegraphics[height=3.5cm]{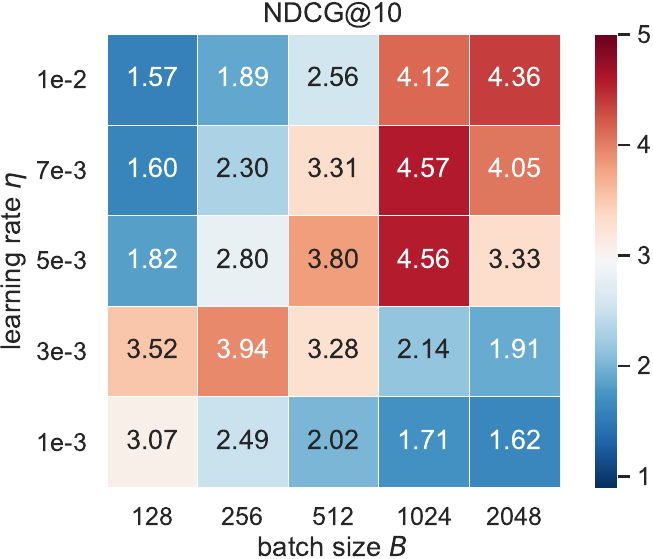}
\caption{parameter sharing}
\label{fig:heat1}
\end{subfigure}
\begin{subfigure}{0.3\textwidth}
\centering
\includegraphics[height=3.5cm]{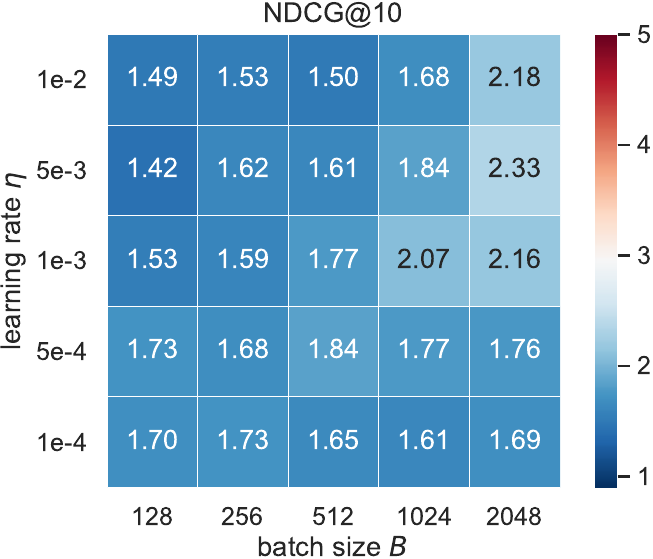}
\caption{w/o parameter sharing}
\label{fig:heat2}
\end{subfigure}
\begin{subfigure}{0.3\textwidth}
\centering
\includegraphics[height=3.5cm]{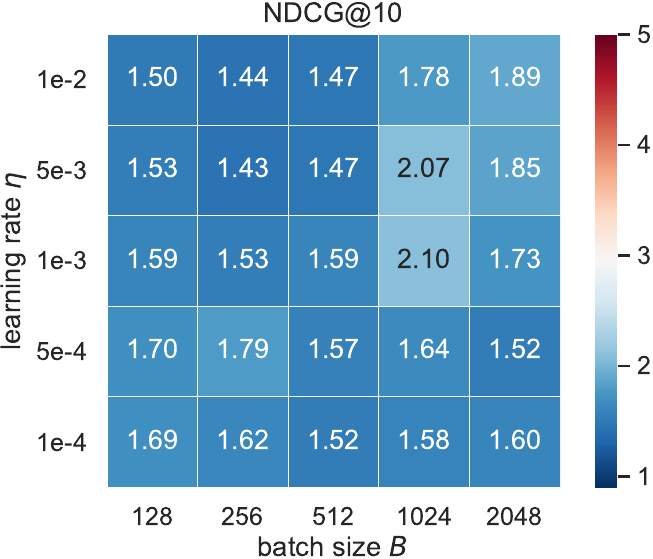}
\caption{halved dimension in (b)}
\label{fig:heat3}
\end{subfigure}
\caption{Numbers are NDCG(\%)@10 (higher is better) of the privately trained model (with $\varepsilon$ set to 5) on MovieLens (\Cref{fig:dataset}). Parameter sharing for the embedding layer yields consistent and significant performance gains over the non-sharing setting in private training. The optimal hyperparameter configuration is always using a large batch size (with a large learning rate).}
\label{fig:embsharing}
\end{figure}

\begin{figure}[ht]
\vspace{-0.3cm}
\centering
\begin{subfigure}{0.48\textwidth}
\centering
\setlength{\abovecaptionskip}{0.cm}
\includegraphics[width=5cm]{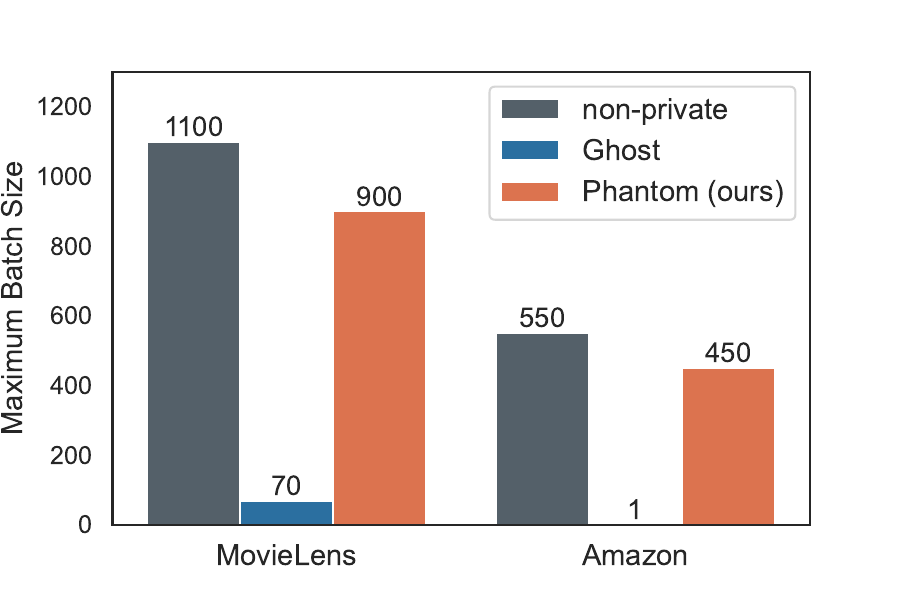}
\caption{Memory efficiency}
\label{fig:Memory}
\end{subfigure}
\begin{subfigure}{0.48\textwidth}
\centering
\setlength{\abovecaptionskip}{0.cm}
\includegraphics[width=5cm]{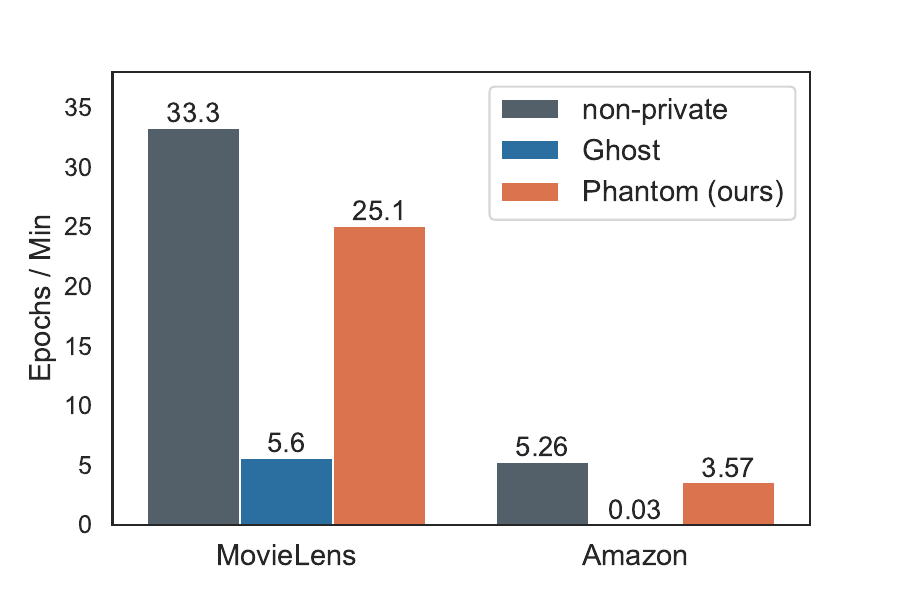}
\caption{Training speed}
\label{fig:Throughput}
\end{subfigure}

\label{fig:phantom}

\caption{Empirical speedup on small models, compared to the Ghost Clipping.  \textbf{Left:} Phantom Clipping is 10-400$\times$ more memory efficient than Ghost Clipping and is almost as efficient as non-private training. \textbf{Right:} Phantom Clipping is 4-100$\times$ faster than Ghost Clipping, having comparable training speed with non-private training.}
\vspace{-0.2cm}
\end{figure}

\noindent \textbf{Empirical Speedup.} 
As we mentioned, in addition to support standard Transformer models with embedding sharing, our Phantom Clipping can achieve great speedup than Ghost Clipping~\cite{li2022large} when the model is relatively small.
This is because the Ghost Clipping is a general method for all layers, not specialized for the embedding layer. However, when the model size is small, the dominating overhead will be at the embedding layer.

We implement our Phantom Clipping based on AWS's fastDP\footnote{\url{https://github.com/awslabs/fast-differential-privacy}} library, which has implemented Ghost Clipping.
We then empirically compare our Phantom Clipping with Ghost Clipping in terms of both memory footprint and training speed on real-world datasets\footnote{Since Ghost Clipping does not support parameter sharing, its results are obtained from training models without embedding sharing. This leads to more model parameters. For a fair comparison, we halve its embedding dimension to $d_E/2$, ending up with a similar number of parameters as in the model with embedding sharing.} (see \Cref{fig:dataset} for details of the datasets).
\Cref{fig:Memory} shows the maximum batch size that can fit into a  Tesla V100 GPU (16 GB of VRAM). It can be seen that our technique is much more memory friendly.
It allows up to $450\times$ larger batch size compared with Ghost Clipping on Amazon, almost as large as those in non-private training.
\Cref{fig:Throughput} shows the training speed on a single Tesla V100 GPU. It allows up to $100\times$ training speedup in practice compared to Ghost Clipping, achieving  0.68$\times$ training speed of the non-private version.

\subsection{Empirical Evaluation of the Re-Attention Mechanism}
\label{apdx:empiricalreattn}

\textbf{Baselines and Implementation Details.} 
We compare our method with vanilla Transformer~\citep{vaswani2017attention} (i.e., the one without Re-Attention Mechanism), vanilla Transformer without parameter sharing, GRU~\citep{cho2014learning}, and LSTM~\citep{hochreiter1997long}. For a fair comparison, embedding sharing is applied for all evaluated methods if not explicitly stated.
The number of epochs is set to 100, where the first 20\% of epochs are used for learning rate warm-up. After that, we linearly decay the learning rate through the remaining epochs.
Following~\citep{bu2023automatic,yang2022normalized}, we normalize the gradients and set the clipping norm $C$ to 1, which eliminates the hyperparameter tuning for clipping norm $C$. For privacy accounting, we fix the total training epochs (iterations) and derive the noise required for each iteration from the preset privacy budget $\varepsilon$. The parameter $\delta$ in DP guarantee is set to $1/\text{size of dataset}$.


\begin{table*}[ht]
\scriptsize
\centering
\caption{Best results (\%) on MovieLens at different privacy levels.}
\begin{tabular}{@{}clcccccc@{}}
\toprule
\multicolumn{2}{c}{DP Guarantee} & \multicolumn{2}{c}{$\varepsilon=5$}   & \multicolumn{2}{c}{$\varepsilon=8$}   & \multicolumn{2}{c}{$\varepsilon=10$} \\ \midrule
\multicolumn{2}{c}{Metric}       & NDCG@10 & \multicolumn{1}{c|}{HIT@10} & NDCG@10 & \multicolumn{1}{c|}{HIT@10} & NDCG@10           & HIT@10           \\ \midrule\midrule
\multicolumn{2}{c}{\textsc{GRU}}  & 2.26 $\pm$ 0.04 & \multicolumn{1}{c|}{4.58 $\pm$ 0.09} & 2.40  $\pm$ 0.03   & \multicolumn{1}{c|}{4.75 $\pm$ 0.20} & 2.81 $\pm$ 0.03 & 5.53 $\pm$ 0.05 \\
\multicolumn{2}{c}{\textsc{LSTM}}   & 2.65 $\pm$ 0.07 & \multicolumn{1}{c|}{5.08 $\pm$ 0.08} & 2.76 $\pm$ 0.03 & \multicolumn{1}{c|}{5.41 $\pm$ 0.06} & 2.95 $\pm$ 0.03 & 5.55 $\pm$ 0.06 \\
\multicolumn{2}{c}{\textsc{Transformer w/o PS}}  & 2.33 $\pm$ 0.05 & \multicolumn{1}{c|}{4.47 $\pm$ 0.07} & 2.56 $\pm$ 0.03    & \multicolumn{1}{c|}{5.11 $\pm$ 0.05} & 2.74 $\pm$ 0.04  & 5.39 $\pm$ 0.08  \\
\multicolumn{2}{c}{\textsc{Transformer (Vanilla)}}  & \textbf{4.57 $\pm$ 0.26} & \multicolumn{1}{c|}{\textbf{8.69 $\pm$ 0.53}} & \textbf{7.05 $\pm$ 0.23}    & \multicolumn{1}{c|}{\textbf{13.17 $\pm$ 0.37}} & \textbf{7.99 $\pm$ 0.21} & \textbf{14.82 $\pm$ 0.38} \\ \midrule \midrule
\multicolumn{2}{c}{\textsc{Ours}} & \textbf{5.88 $\pm$ 0.24}  & \multicolumn{1}{c|}{\textbf{11.13 $\pm$ 0.43}}     &  \textbf{7.70 $\pm$ 0.26} & \multicolumn{1}{c|}{\textbf{14.31 $\pm$ 0.37}}     & \textbf{8.42 $\pm$ 0.22} & \textbf{15.40 $\pm$ 0.32} \\ \midrule
\multicolumn{2}{c}{Relative Improvement} & \textcolor{darkgreen}{\textbf{29\%$\uparrow$}} &  \multicolumn{1}{c|}{\textcolor{darkgreen}{\textbf{28\%$\uparrow$}}}   & \textcolor{darkgreen}{\textbf{9.2\%$\uparrow$}} &  \multicolumn{1}{c|}{\textcolor{darkgreen}{\textbf{8.7\%$\uparrow$}}}     & \textcolor{darkgreen}{\textbf{5.4\%$\uparrow$}} & \textcolor{darkgreen}{\textbf{3.9\%$\uparrow$}} \\
\bottomrule
\end{tabular}
\label{tb:accml}
\end{table*}

\begin{table*}[ht]
\scriptsize
\centering
\vspace{-0.2cm}
\caption{Best results (\%) on Amazon at different privacy levels.}
\begin{tabular}{@{}clcccccc@{}}
\toprule
\multicolumn{2}{c}{DP Guarantee} & \multicolumn{2}{c}{$\varepsilon=5$}   & \multicolumn{2}{c}{$\varepsilon=8$}   & \multicolumn{2}{c}{$\varepsilon=10$} \\ \midrule
\multicolumn{2}{c}{Metric}       & NDCG@10 & \multicolumn{1}{c|}{HIT@10} & NDCG@10 & \multicolumn{1}{c|}{HIT@10} & NDCG@10           & HIT@10           \\ \midrule\midrule
\multicolumn{2}{c}{\textsc{GRU}}  & 1.13 $\pm$ 0.02 & \multicolumn{1}{c|}{2.46 $\pm$ 0.03} & 1.33 $\pm$ 0.02    & \multicolumn{1}{c|}{2.22 $\pm$ 0.02} & 1.47 $\pm$ 0.03 & 2.48 $\pm$ 0.02 \\

\multicolumn{2}{c}{\textsc{LSTM}}   & 1.19 $\pm$ 0.01  & \multicolumn{1}{c|}{2.46 $\pm$ 0.04} & 1.23 $\pm$ 0.01  & \multicolumn{1}{c|}{2.46 $\pm$ 0.04} &  1.34 $\pm$ 0.01 &  2.51 $\pm$ 0.02 \\

\multicolumn{2}{c}{\textsc{Transformer w/o PS}}  & 1.16 $\pm$ 0.01 & \multicolumn{1}{c|}{2.36 $\pm$ 0.01} & 1.20 $\pm$ 0.02    & \multicolumn{1}{c|}{2.38 $\pm$ 0.01} & 1.40 $\pm$ 0.01  & 2.47 $\pm$ 0.02  \\

\multicolumn{2}{c}{\textsc{Transformer (Vanilla)}}  & \textbf{1.37 $\pm$ 0.04} & \multicolumn{1}{c|}{\textbf{2.47 $\pm$ 0.10}} & \textbf{1.54 $\pm$ 0.03}    & \multicolumn{1}{c|}{\textbf{2.77 $\pm$ 0.07}} & \textbf{1.57 $\pm$ 0.03}  & \textbf{2.83 $\pm$ 0.08}  \\ \midrule \midrule

\multicolumn{2}{c}{\textsc{Ours}} & \textbf{1.64 $\pm$ 0.01}  & \multicolumn{1}{c|}{\textbf{3.01 $\pm$ 0.01}}     & \textbf{1.98 $\pm$ 0.05} & \multicolumn{1}{c|}{\textbf{3.70 $\pm$ 0.15}}     & \textbf{1.99 $\pm$ 0.04} & \textbf{3.73 $\pm$ 0.11} \\ \midrule

\multicolumn{2}{c}{Relative Improvement} & \textcolor{shilv}{\textbf{20\%$\uparrow$}} &  \multicolumn{1}{c|}{\textcolor{darkgreen}{\textbf{22\%$\uparrow$}}}   & \textcolor{shilv}{\textbf{28\%$\uparrow$}} &  \multicolumn{1}{c|}{\textcolor{darkgreen}{\textbf{34\%$\uparrow$}}}     & \textcolor{darkgreen}{\textbf{27\%$\uparrow$}} &\textcolor{darkgreen}{\textbf{31\%$\uparrow$}} \\
\bottomrule
\end{tabular}
\label{tb:accam}
\end{table*}

\Cref{tb:accml} and \Cref{tb:accam} show the best\footnote{Strictly speaking, the process of hyperparameter tuning would cost privacy budget~\citep{papernot2022hyperparameter}, but is mainly of theoretical interest. We perform grid search on learning rate $\in\{10^{-3}, 3\times10^{-3}, 5\times10^{-3},  7\times10^{-3},  9\times10^{-3}\}$ and batch size $\in\{256, 512, 1024, 2048, 4096\}$ for each method, ensuring fair comparison.} NDCG@10 and HIT@10 for all the methods on MovieLens and Amazon. 
The vanilla Transformer outperforms all other baselines, reaffirming its dominance in sequential data modeling due to the Attention Mechanism. Our Re-Attention Mechanism further boosts the performance by around 20\% on average. Notably, under a low privacy budget ($\varepsilon=5$), our method achieves a relative improvement of around 25\%, demonstrating its efficacy in attenuating attention distraction during private training. On MovieLens, as expected, the performance gain increases with decreasing privacy budget $\varepsilon$, i.e., increasing noise strength during training. This is because larger noise corresponds to more severe attention distraction, which better highlights the Re-Attention Mechanism's advantage. 
However, on Amazon, our method achieves a smaller relative improvement at $\varepsilon=5$ than at $\varepsilon=10$. We suspect that this is due to the differences of the two datasets in terms of sparsity (i.e., $1-$density in \Cref{fig:dataset}) as well as the inherent hardness of training Transformer~\citep{zhang2019improving,xu2020optimizing,huang2020improving} and the overwhelming DP noise.

\Cref{fig:convergence} shows the model accuracy every five epochs during training. Evidently, the training dynamics of the vanilla Transformer, impacted by attention distraction, can suffer from high variance and/or substantial fluctuation, especially on Amazon. In contrast, our method enjoys faster and smoother convergence, highlighting its training stability under differential privacy.

\begin{figure}[ht]

\centering
\begin{subfigure}{0.24\textwidth}
\centering
\includegraphics[height=2.7cm]{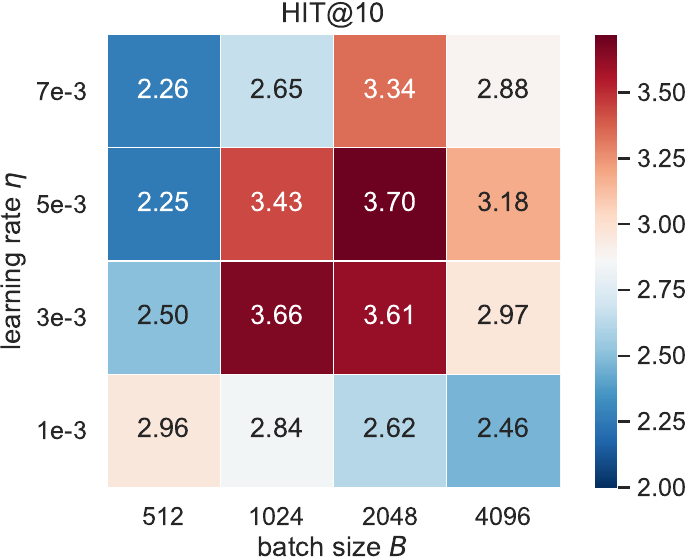}
\caption{Ours}
\label{fig:grad_ndcg_ama_dp}
\end{subfigure}
\begin{subfigure}{0.24\textwidth}
\centering
\includegraphics[height=2.7cm]{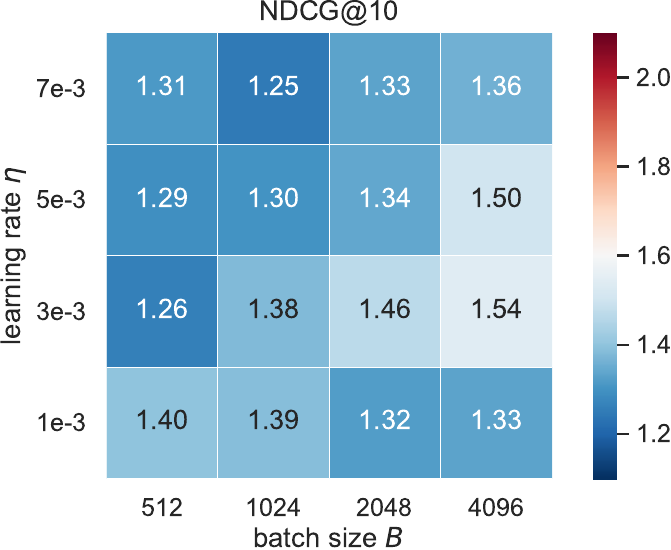}
\caption{Vanilla}
\label{fig:grad_ndcg_ama_va}
\end{subfigure}
\begin{subfigure}{0.24\textwidth}
\centering
\includegraphics[height=2.7cm]{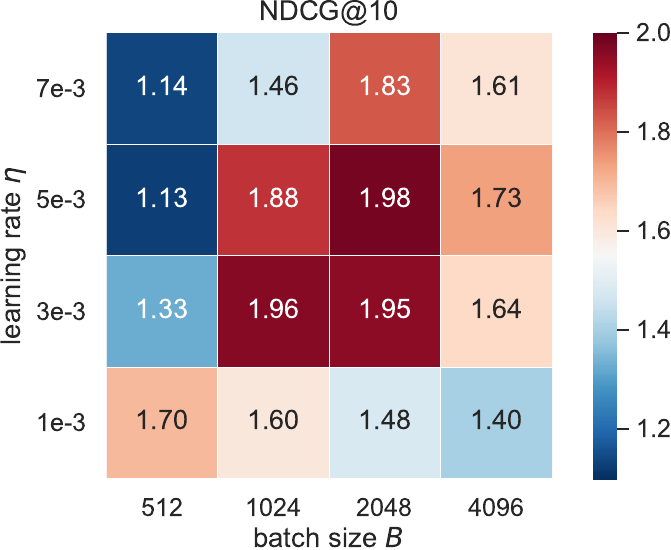}
\caption{Ours}
\label{fig:grad_hit_ama_dp}
\end{subfigure}
\begin{subfigure}{0.24\textwidth}
\centering
\includegraphics[height=2.7cm]{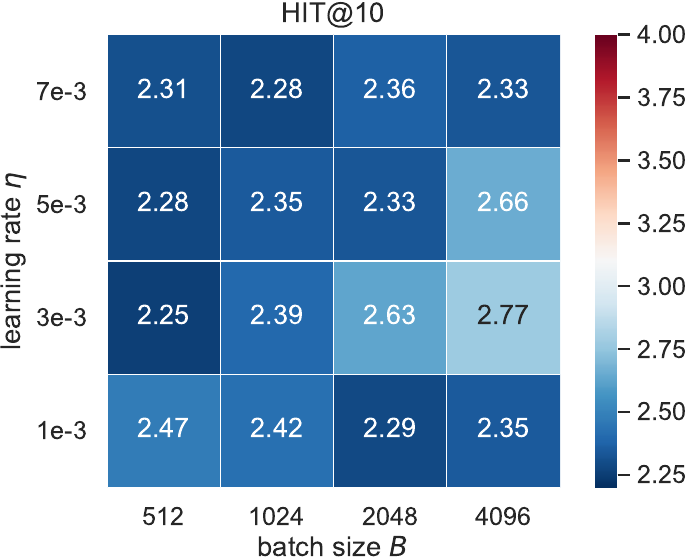}
\caption{Vanilla}
\label{fig:grad_hit_ama_va}
\end{subfigure}
\caption{Results of grid search for hyperparameter tuning on Amazon with privacy budget $\varepsilon=8$.}
\label{fig:gird_exp}
\end{figure}

To study the robustness and sensitivity with respect to the hyperparameters of our method, \Cref{fig:gird_exp} shows the results of hyperparameter tuning via grid search\footnote{Rather than run an additional differentially private algorithm to report a noisy max (or argmax)~\citep{papernot2022hyperparameter}, we opt for this practice of directly displaying all results due to its transparency and comprehensiveness.}. For \textit{reasonable} (along the main diagonal~\citep{tramerdifferentially}) hyperparameter configurations, our method significantly and consistently outperforms the vanilla Transformer.

We present the visualization of the attention score map in \Cref{fig:exp_attn}.
\label{apdx:exp:distraction}
\begin{figure}[ht!]
\centering

\includegraphics[width=12cm]{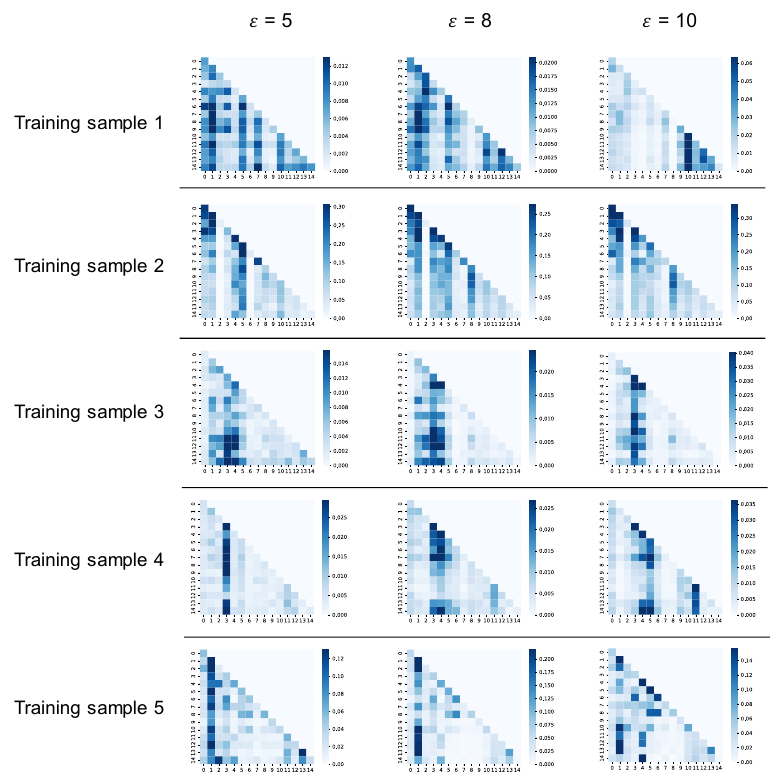}

\caption{Visualization of attention score with varying $\varepsilon$ on MovieLens for five training samples.}
\label{fig:exp_attn}
\end{figure}

We randomly draw five sentences from the dataset MovieLens and visualize their attention matrices with $\varepsilon=5, 8, 10$ (\Cref{fig:exp_attn}). Recall that as the value of $\varepsilon$
increases (indicating a lesser amount of noise), the accuracy of the attention score improves (and consequently, the model performance) will be. With a larger $\varepsilon$ (i.e., 10), the attention score distribution more closely aligns with the ground truth.
It is clear from \Cref{fig:exp_attn} that due to the attention distraction caused by DP noise, when $\varepsilon$ is low, some tokens receive a larger amount of attention than they should have (compared with $\varepsilon=10$), thereby resulting in suboptimal model performance.

\subsection{Additional Experiments on NLP Task}
The main focus of this paper is to, in the context of the introduced modular treatment, initiate a systematic study on the reduction from DP Transformer to (vanilla) DP neural nets (Figure 1). The experiments in the work intend to corroborate the theoretical analysis empirically.

However, when it comes to the natural language processing (NLP) tasks, it inevitably involves the use of large models and potentially more reasonable ways to exploit public data while maintaining provable privacy, which itself is already non-trivial and largely open. Since the focus of this work is on initiating a systematic study on DP Transformers (as a basic primitive rather than a concrete application), we leave that more complicated scenario (and also on various domains) for future work. With that being said, we conducted an experiment of language modeling on TinyShakespeare\footnote{https://paperswithcode.com/dataset/tinyshakespeare} dataset, where a Transformer model is trained from scratch. 

We use the standard Transformer encoder described in~\citep{vaswani2017attention}. The model dimension is set to 384. The number of heads in the Attention Mechanism is set to 6. The number of Transformer blocks is set to 6.
The number of epochs is set to 2000. The batch size is chosen from $\{256, 512, 1024, 2048, 4096\}$. The learning rate is chosen from $\{10^{-5}, \times10^{-4}, 5\times10^{-4}, 10^{-3}\}$. The dropout rate is 0.2. We use the Adam optimizer.

\begin{figure}[ht!]
\centering

\begin{subfigure}{0.4\textwidth}
\centering
\setlength{\abovecaptionskip}{0.cm}
\includegraphics[width=7cm]{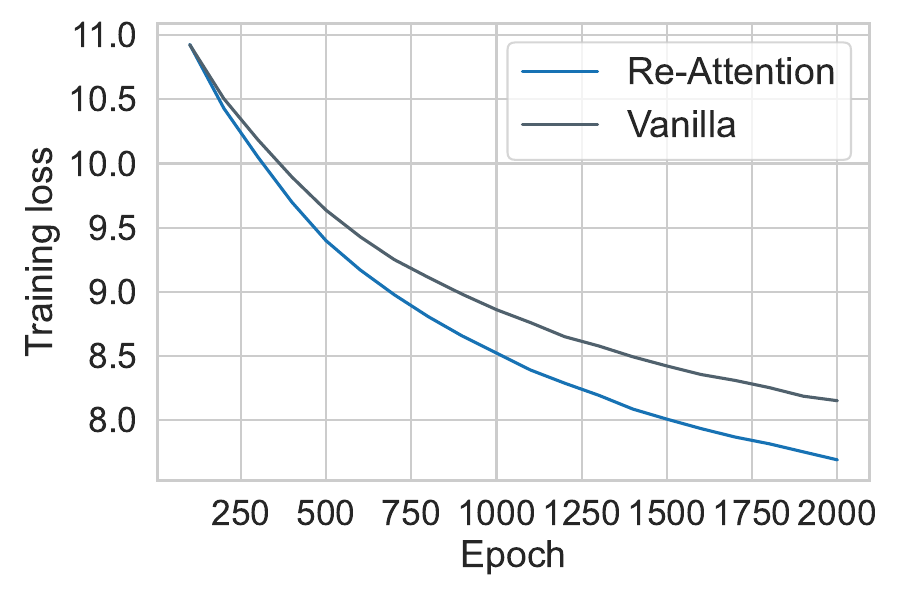}
\caption{Training loss ($\varepsilon=10$)}
\end{subfigure}
\begin{subfigure}{0.4\textwidth}
\centering
\setlength{\abovecaptionskip}{0.cm}
\includegraphics[width=7cm]{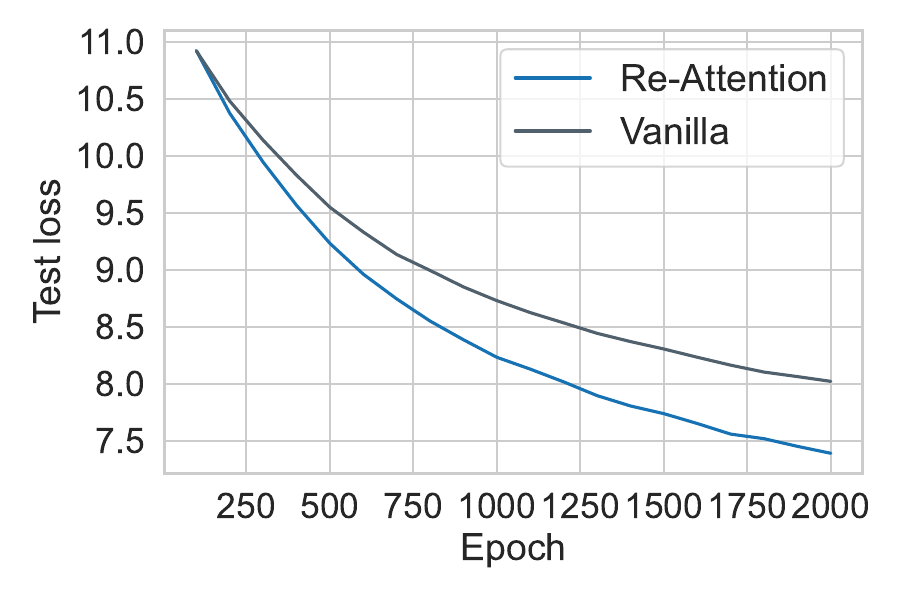}
\caption{Test loss ($\varepsilon=10$)}
\end{subfigure}

\begin{subfigure}{0.4\textwidth}
\centering
\setlength{\abovecaptionskip}{0.cm}
\includegraphics[width=7cm]{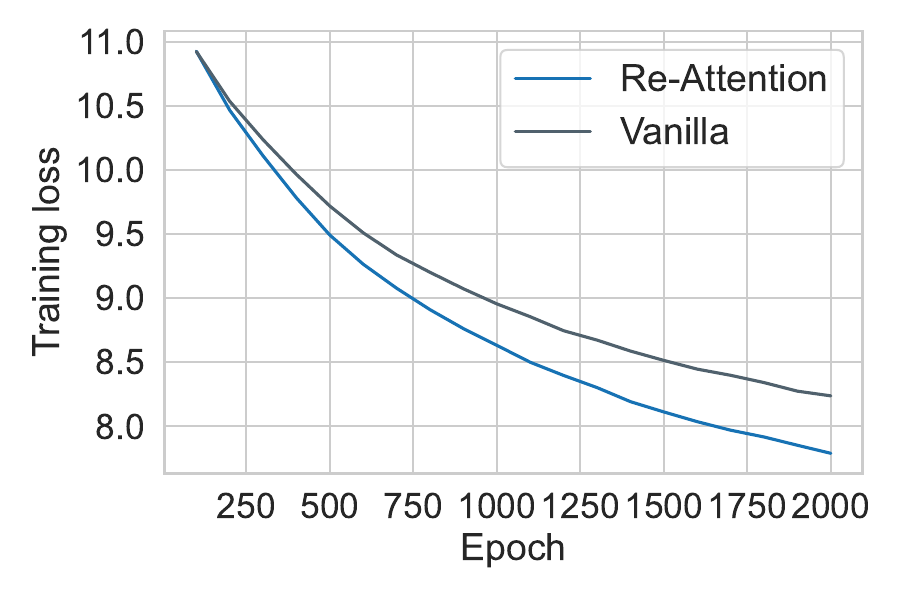}
\caption{Training loss ($\varepsilon=8$)}
\end{subfigure}
\begin{subfigure}{0.4\textwidth}
\centering
\setlength{\abovecaptionskip}{0.cm}
\includegraphics[width=7cm]{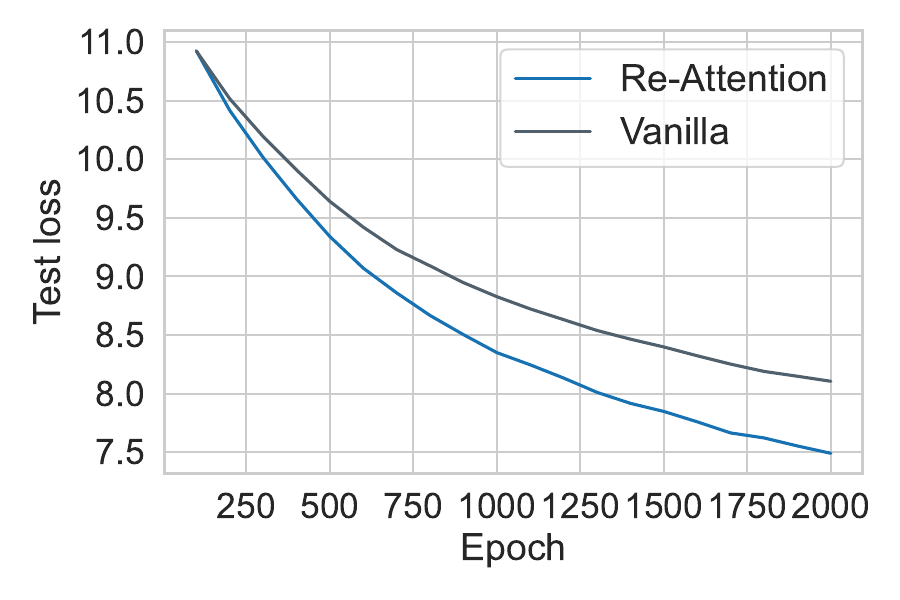}
\caption{Test loss ($\varepsilon=8$)}
\end{subfigure}

\begin{subfigure}{0.4\textwidth}
\centering
\setlength{\abovecaptionskip}{0.cm}
\includegraphics[width=7cm]{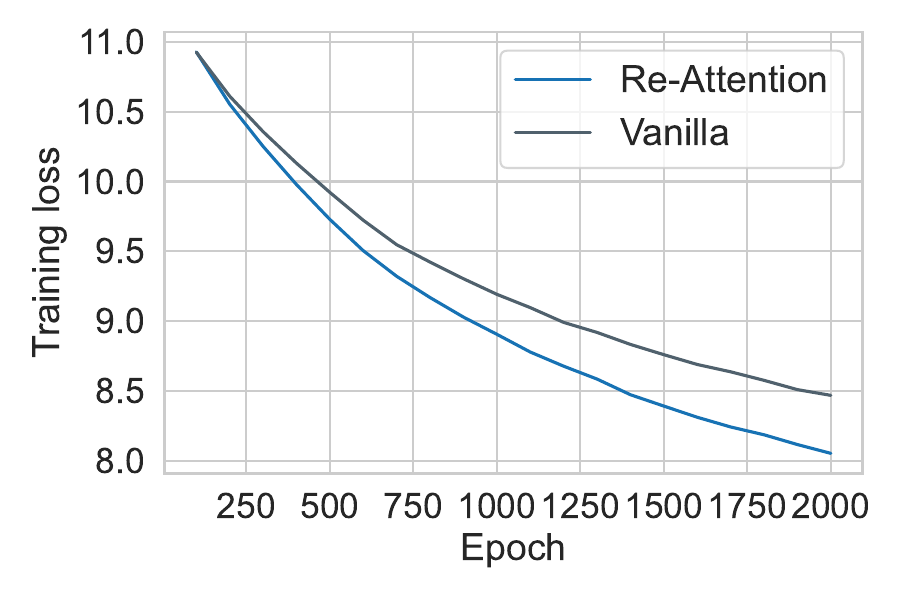}
\caption{Training loss ($\varepsilon=5$)}
\end{subfigure}
\begin{subfigure}{0.4\textwidth}
\centering
\setlength{\abovecaptionskip}{0.cm}
\includegraphics[width=7cm]{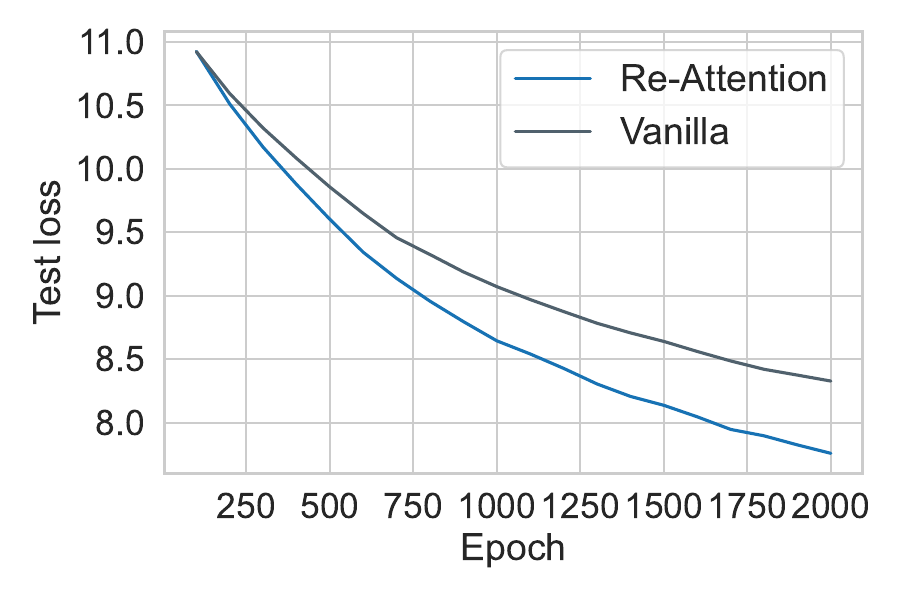}
\caption{Test loss ($\varepsilon=5$)}
\end{subfigure}

\label{fig:nlp}
\end{figure}

\end{document}